\documentclass{article}

\usepackage{PRIMEarxiv}

\usepackage[utf8]{inputenc}
\usepackage[T1]{fontenc}
\usepackage{hyperref}
\usepackage{url}
\usepackage{booktabs}
\usepackage{amsfonts}
\usepackage{amsmath,amssymb,amsthm}
\usepackage{nicefrac}
\usepackage{microtype}
\usepackage{fancyhdr}
\usepackage{graphicx}
\usepackage{tabularx}
\usepackage{xcolor}
\usepackage{multirow}
\usepackage{enumitem}
\usepackage{natbib}
\usepackage{booktabs}
\graphicspath{{media/}}

\newcommand{\method}{\textsc{NSR}}

\newtheorem{proposition}{Proposition}

\title{Task-Relevant Null-Space Residuals for Non-Injective Neural Mappings}

\author{
  \textbf{
    Bizu Feng\textsuperscript{1,2,3},
    Zhimu Yang\textsuperscript{4},
    Shuming Wang\textsuperscript{2,3},
    Yuan Cheng\textsuperscript{1,2}
  }
  \\
  \textbf{
    Shaode Yu\textsuperscript{4},
    Xiaojun Qian\textsuperscript{1},
    Zixin Hu\textsuperscript{1,2,*}
  }
  \\[0.7em]
  {\small
    \textsuperscript{1}Institute of Artificial Intelligence Innovation and Industry,
    Fudan University, Shanghai, China
  }
  \\
  {\small
    \textsuperscript{2}Shanghai Academy of AI for Science,
    Shanghai, China
  }
  \\
  {\small
    \textsuperscript{3}Human Phenome Institute,
    Fudan University, Shanghai, China
  }
  \\
  {\small
    \textsuperscript{4}School of Information and Communication Engineering,
    Communication University of China, Beijing, China
  }
  \\[0.5em]
  {\footnotesize
    \texttt{bzfeng25@m.fudan.edu.cn},
    \texttt{muyuzhierchengse@gmail.com}
  }
  \\
  {\footnotesize
    \texttt{smwang@m.fudan.edu.cn},
    \texttt{cheng\_yuan@fudan.edu.cn}
  }
  \\
  {\footnotesize
    \texttt{yushaodecuc@cuc.edu.cn},
    \texttt{qianxiaojun@fudan.edu.cn}
  }
  \\[0.4em]
  {\small
  \textsuperscript{*}Correspondence:
  \texttt{huzixin@fudan.edu.cn}
  }
}

\date{}

\begin{document}
\maketitle

% ============================================================
\begin{abstract}
Non-injective mappings in neural networks map distinct inputs to the same representation, thereby implicitly inducing equivalence relations in the input space. However, the input differences eliminated by these mappings may still be required by downstream tasks, creating a mismatch between operator-induced indistinguishability and task-required distinctions. For non-injective linear operators realized in the current forward pass, their null spaces exactly characterize these invisible input variations. We propose Task-Relevant Null-Space Residuals (NSR), a general residual framework for non-injective linear mappings. NSR combines null-space component extraction from pre-mapping representations, member-level encoding and gating, and application-specific integration to exploit potentially task-relevant information under downstream supervision while preserving the original aggregation or merging rules. We evaluate NSR in two structurally different settings: token merging and graph aggregation. In token merging, NSR achieves higher semantic segmentation performance
than the corresponding compressed baselines in 34 out of 36 evaluated configurations, with a maximum observed gain of 31.51 mIoU points under strong compression. In graph aggregation, NSR achieves 100\% training accuracy on Tree-NeighborsMatch at depths \(d=2\text{--}6\) across three backbones, alongside gains on heterophilic node classification and molecular graph regression. Together, these results support null-space residuals as a practical complement to non-injective linear mappings, enabling downstream models to learn from input distinctions invisible in the original operator's output.
\end{abstract}

% ============================================================
% ============================================================
\section{Introduction}
\label{sec:introduction}

Many representation transformations in neural networks are non-injective:
distinct inputs can be mapped to the same output, making them
indistinguishable to any subsequent computation that depends only on that
output. Such indistinguishability does not necessarily hinder the task;
when these inputs require the same prediction, ignoring their differences
is appropriate. The issue is that inputs identified by the mapping may
still correspond to different task targets~\cite{jacobsen2018excessive}.
Therefore, a non-injective mapping not only transforms representations
but also determines, through its structure, which inputs are
indistinguishable in its output. These operator-induced input equivalences
need not be compatible with downstream task requirements
(Fig.~\ref{fig:motivation}).
Making a distinction invisible to an operator does not mean that
the task no longer requires it.

\begin{figure}[h]
    \centering
    \includegraphics[width=1.0\linewidth]{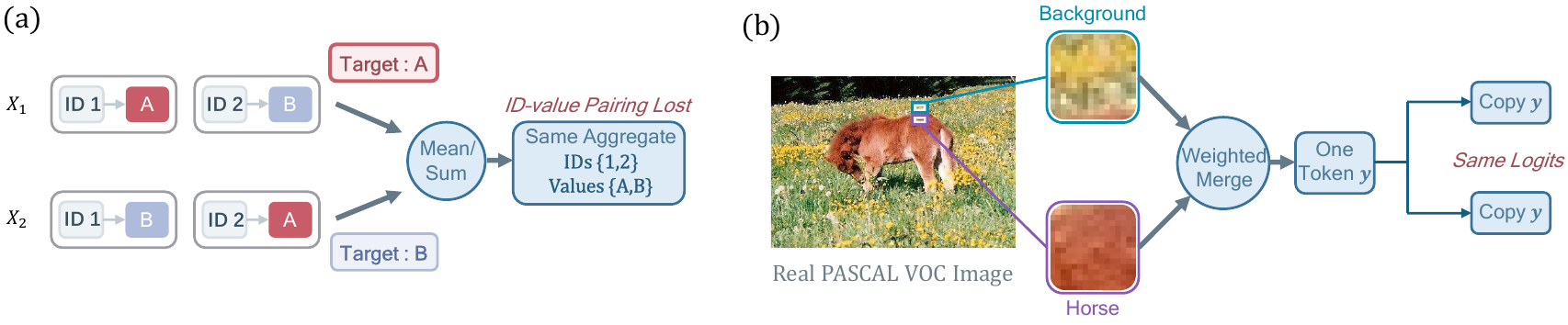}
    \caption{
        Operator-induced indistinguishability need not align with
        downstream task requirements: distinct inputs may produce
        the same mapped representation while requiring different
        task outputs.
    }
    \label{fig:motivation}
\end{figure}

Starting from this general problem, we focus on non-injective linear
mappings, including computational steps in adaptive operations that take
a linear form once the current member groups and weights have been
determined. The key reason for this scope is that, for a fixed linear
operator, two inputs produce the same output if and only if their
difference lies in its null space. Thus, the null space provides an
exact description of input variations invisible to the operator.
For operations with input-dependent grouping or weights, this
characterization is conditioned on the linear operator realized in the
current forward pass. Linear neighborhood aggregation in graph neural
networks and weighted token merging with the current groups and weights
fixed are two structurally different instances of this
setting~\cite{kipf2016semi,bolya2023tome}.

This perspective leads to the question studied in this work:
Can we retain the original mapping rule while providing the
model with a separate pathway to learn to exploit input variations
invisible in the original operator's output?

To this end, we propose Task-Relevant Null-Space Residuals (NSR),
a general residual framework for non-injective linear mappings,
with its overall structure shown in Fig.~\ref{fig:overview}.
NSR uses the current operator to explicitly extract null-space
components from pre-mapping representations as candidate
complementary signals.
Within a task-supervised complementary pathway, residual encoding,
member-wise gating, and application-specific integration jointly
transform these components into task-adaptive updates, enabling
downstream models to exploit distinctions invisible in the
original operator's output.
In graph aggregation and token merging, NSR instantiates this
pathway while preserving the original aggregation or merging rules.
\begin{figure}[t]
    \centering
    \includegraphics[width=1.0\linewidth]{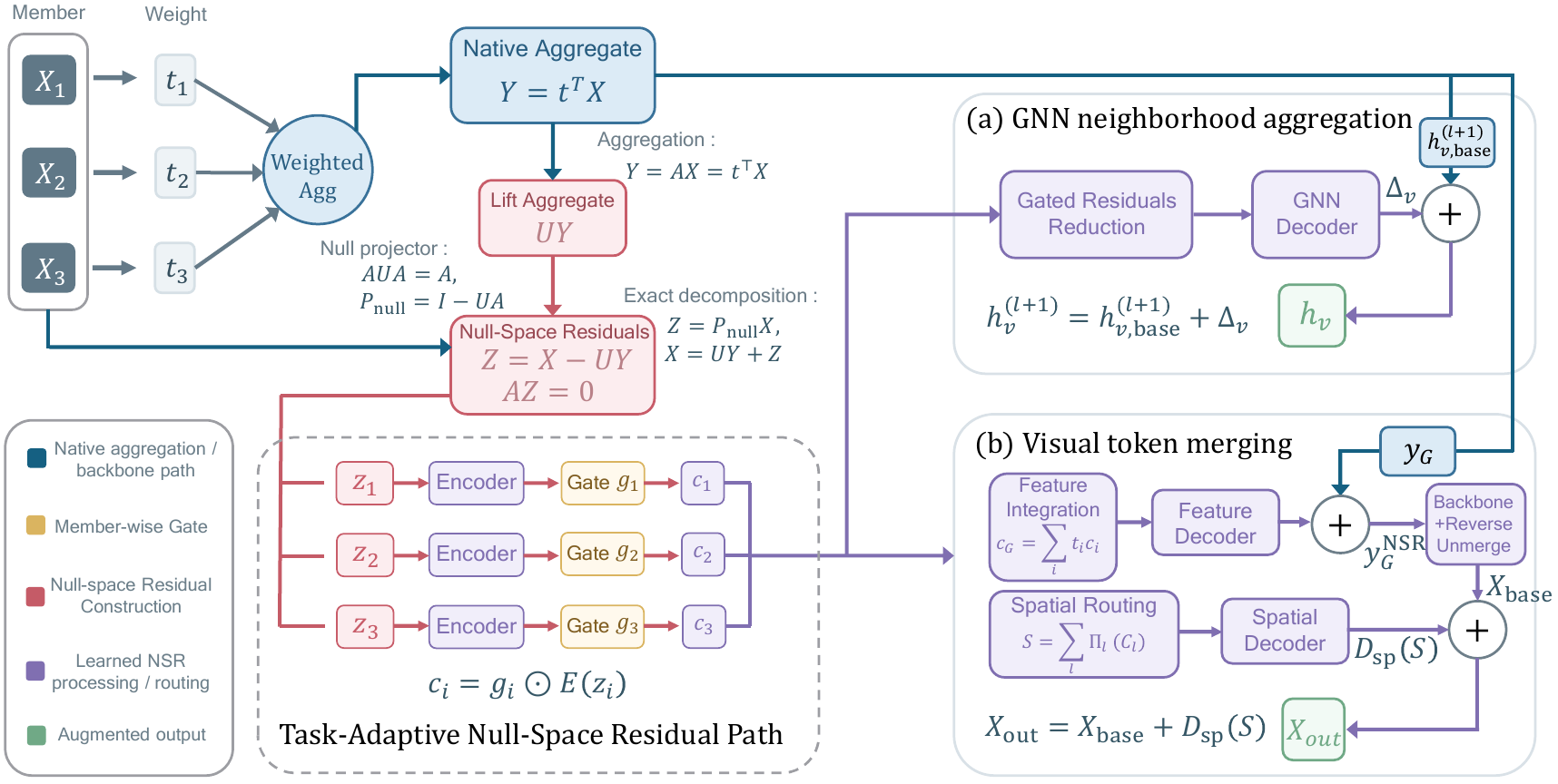}
    \caption{
        Overview of NSR. Null-space residuals are extracted from
        pre-mapping representations based on the current linear
        operator and transformed into usable complementary information
        through a learnable branch under task supervision.
        The original mapping branch retains its computation rule.
    }
    \label{fig:overview}
\end{figure}
Our main contributions are as follows:
\begin{itemize}
    \item \textbf{Operator--task mismatch.}
    We characterize the mismatch between the input equivalence classes
    induced by the currently realized non-injective linear operator
    and the distinctions required by downstream tasks, providing
    an operator-level basis for constructing complementary information.
    \item \textbf{Null-space residual framework.}
We propose NSR, a general task-supervised residual framework
for non-injective linear mappings. The framework is not limited
to null-space extraction but comprises operator-defined
null-space component extraction, member-level encoding and
gating, and application-specific integration.
    \item \textbf{Experimental evidence across structures and tasks.}
Token-merging experiments cover three compression methods, three
semantic-segmentation datasets, and multiple compression levels, with
NSR achieving higher mIoU than the corresponding compressed baselines
in 34 of 36 evaluated configurations.
In graph aggregation, the complete NSR pathway achieves 100\%
training accuracy on Tree-NeighborsMatch at depths
\(d=2\text{--}6\) across three backbones, alongside gains on
heterophilic node classification and molecular graph regression.
\end{itemize}

% ============================================================
\section{Related Work}
\label{sec:related}

\textbf{Representation invariance and null-space methods.}
Invariance to input variations does not always align with task
requirements. Neural networks can be overly sensitive to task-irrelevant
variations while exhibiting excessive invariance to task-relevant
changes~\citep{jacobsen2018excessive}. At the architectural level, i-RevNet maintains input reconstructability
through invertible intermediate
transformations~\citep{jacobsen2018irevnet}; LiftPool uses an invertible
subband decomposition and employs the detail subbands produced during
downsampling for subsequent upsampling~\citep{zhao2021liftpool}.
Null spaces have also been used in learning-based inverse problem
solving: null-space networks constrain learned reconstruction
corrections to the null space of the forward operator to preserve data
consistency~\citep{schwab2019deep}. NSR retains the existing non-injective main mapping without requiring
it to be made invertible; its null-space components are extracted
directly from accessible pre-mapping representations rather than
inferred from mapped outputs. Here, the null space defines the source
of candidate complementary signals, while a separate learnable branch
learns how to exploit this information under downstream task supervision.

\textbf{Graph aggregation and residual connections.}
Neighborhood aggregation is a core operation through which
message-passing GNNs construct node representations~\citep{gilmer2017neural}.
GCN uses degree-normalized weighted summation~\citep{kipf2016semi};
GraphSAGE introduces several neighborhood aggregation schemes, with its mean variant combining the neighborhood mean with the central
node representation~\citep{hamilton2017graphsage};
GIN uses sum aggregation and a multilayer perceptron, relating the
discriminative power of multiset aggregation to the expressive power
of GNNs~\citep{xu2019powerful}. More generally, Deep Sets studies permutation-invariant functions
through shared element-wise transformations and
aggregation~\citep{zaheer2017deepsets}.
PNA combines multiple aggregation statistics with degree-dependent
scalers to enrich neighborhood representations~\citep{corso2020pna}.
Another line of work focuses on the propagation and combination of
representations across layers: Jumping Knowledge combines node
representations from different propagation depths~\citep{xu2018jumping},
while GCNII supports deeper graph convolutional networks through
initial residual connections and identity
mappings~\citep{chen2020simple}. NSR differs in how its complementary signal is defined: it retains
the backbone's original aggregation branch and uses the null space
of the current local aggregation operator to extract operator-invisible
member-wise residuals from pre-aggregation representations.
These residuals are then transformed into node updates through
task-adaptive processing.

\textbf{Token merging in vision transformers.}
ViT represents an image as a sequence of embedded image
patches~\citep{dosovitskiy2021vit}.
Token merging summarizes multiple tokens into representative tokens
to shorten the sequence used in subsequent computation.
ToMe progressively merges similar tokens using lightweight
matching~\citep{bolya2023tome};
PiToMe introduces an energy score to select merging candidates,
prioritizing the preservation of relatively distinctive or isolated
tokens~\citep{tran2024pitome};
ALGM adopts a local-then-global merging strategy~\mbox{\citep{norouzi2024algm}};
DTEM uses a separate learnable embedding module to produce dedicated
representations for token matching~\citep{lee2024dtem};
MPM performs mean merging over mutually nearest-neighbor token
pairs~\citep{rave2026mpm}. For dense prediction tasks such as semantic segmentation, merged
representations also need to be mapped back to their original spatial
positions. ALGM and MPM copy merged representations back to member
positions using recorded merge
correspondences~\mbox{\citep{norouzi2024algm,rave2026mpm}}.
This operation restores the original token count and spatial
arrangement, but copying itself does not restore the original
representational differences among members of the same group. NSR focuses on member-wise distinctions eliminated by weighted
merging once the current groups and weights have been determined,
while preserving the original selection, matching, and
weight-computation rules.
The same null-space residual source is used through feature and
spatial paths to provide complementary information for subsequent
token computation and dense prediction at the original positions,
respectively.

% ============================================================
\section{Null-Space Residuals}
\label{sec:nsr}

We characterize operator-induced indistinguishability,
construct complementary null-space residuals, and instantiate
their task-supervised processing for graph aggregation
and token merging.

\subsection{Operator-Induced Indistinguishability and Task Relevance}
\label{sec:nsr_indistinguishability}

Let \(F:\mathcal{X}\rightarrow\mathcal{C}\) map inputs in
\(\mathcal{X}\) to representations in \(\mathcal{C}\).
If distinct inputs \(x_1,x_2\in\mathcal{X}\) satisfy
\(F(x_1)=F(x_2)\), any subsequent computation depending
only on \(F(x)\) cannot distinguish them.
This becomes a limitation when the downstream task requires
that distinction; not all collapsed distinctions are task-relevant.

We focus on linear mappings and operations that take a linear
form once the current groups and weights are fixed.
For a fixed linear operator \(A\) and inputs \(x_1,x_2\)
in its domain,
\begin{equation}
    Ax_1=Ax_2
    \quad\Longleftrightarrow\quad
    x_1-x_2\in\ker(A),
    \label{eq:indistinguishability}
\end{equation}
where \(\ker(A)\) is the null space of vectors mapped to zero
by \(A\), exactly characterizing operator-invisible input variations.
For input-dependent grouping or weights, all null-space statements
are conditioned on the operator \(A\) realized in the current
forward pass.

\subsection{Null-Space Complementarity}
\label{sec:nsr_complementarity}

Let \(X\in\mathbb{R}^{n\times d}\) contain \(n\) member
representations of dimension \(d\), and let
\(A\in\mathbb{R}^{k\times n}\) denote the current linear
operator, with output \(Y=AX\).
We choose a linear lift \(U\in\mathbb{R}^{n\times k}\)
satisfying \(AUA=A\)~\cite{benisrael2003generalized}, and define the null-space component as
\begin{equation}
    Z=(I_n-UA)X,
    \label{eq:nsr_definition}
\end{equation}
where \(I_n\) is the \(n\times n\) identity matrix and
\(U\) maps the output representation back to the member space.

\begin{proposition}[Null-space complementarity]
\label{prop:nsr_complementarity}
The component defined above satisfies
\begin{equation}
    AZ=0,
    \qquad
    X=UY+Z.
    \label{eq:nsr_complementarity}
\end{equation}
\end{proposition}

\begin{proof}
Since \(AUA=A\), we have \(AZ=(A-AUA)X=0\).
Moreover, the definition directly gives
\(UY+Z=UAX+(I_n-UA)X=X\).
\end{proof}

Thus, every feature channel of \(Z\) lies in \(\ker(A)\).
For given \(A\) and \(U\), \(Y\) and \(Z\)
uniquely determine \(X\); any auxiliary signal \(H(X)\) can
thus be written as \(H(UY+Z)\).
At fixed \(Y\), member differences are exactly differences in \(Z\).
The lift \(U\) may depend on the operator and need not induce
an orthogonal projection.
NSR uses this decomposition to expose operator-invisible
differences for task-supervised complementary processing.

\subsection{Task-Supervised Complementary Pathway}
\label{sec:nsr_task_adaptive}

Let \(z_i\in\mathbb{R}^{d}\) denote the \(i\)-th member
residual. NSR constructs its code as
\begin{equation}
    c_i=g_i\odot E(z_i),
    \label{eq:nsr_member_code}
\end{equation}
where \(E:\mathbb{R}^{d}\rightarrow\mathbb{R}^{r}\)
is a residual encoder with code dimension \(r\),
\(g_i\in\mathbb{R}^{r}\) is a learnable gate generated
from the member residual and its context, and
\(\odot\) denotes element-wise multiplication.

Member-level processing precedes reaggregation.
For a linear transformation \(B\in\mathbb{R}^{d\times r}\)
shared across members, reaggregation with the original
operator gives \(A(ZB)=(AZ)B=0\).
Member-dependent gating can alter their relative contributions,
so the processed codes need not satisfy the same
cancellation constraint.

Residual extraction specifies the source of candidate complementary
information, while member-level processing and application-specific
integration make this source usable for downstream prediction.
Together, these stages constitute the complete NSR method.
Depending on the application, the codes are aggregated or spatially
routed and mapped to complementary updates, with the branch trained
under the downstream task loss.
The null-space properties and exact decomposition apply to the
complete residual $Z$ before encoding; member and contextual
representations may additionally condition its processing through
the gates.

\subsection{NSR for Graph Aggregation}
\label{sec:nsr_graph}

At layer \(l\), consider the \(K_v\) members participating
in the current aggregation at node \(v\).
The NSR branch generates messages using a linear map
shared across members and stacks them row-wise as
\(M_v\in\mathbb{R}^{K_v\times d}\).
Let \(t_v\in\mathbb{R}^{K_v}\) denote the aggregation
coefficients actually used by the backbone.
The corresponding local operator is \(A_v=t_v^\top\).

For \(t_v\neq0\), choosing
\(U_v=t_v/(t_v^\top t_v)\) gives \(A_vU_v=1\).
The resulting residual is
\begin{equation}
    Z_v
    =
    \left(
        I_{K_v}
        -
        \frac{t_vt_v^\top}{t_v^\top t_v}
    \right)M_v,
    \qquad
    t_v^\top Z_v=0,
    \label{eq:graph_nsr}
\end{equation}
which projects along the member dimension and extracts
message variations invisible to the current local aggregation.
The member set and coefficients follow the corresponding
backbone: GCN uses normalized adjacency coefficients
including self-loops; GraphSAGE uses neighborhood-mean
weights, with its separate root path handled by the backbone;
and GIN includes neighbor terms and a root term weighted
by \(1+\epsilon^{(l)}\).

The graph branch directly processes residuals in the message
space, taking \(E\) to be the identity map and \(r=d\).
Let \(z_{v,i}\) and \(g_{v,i}\) denote the member residual
and its gate, respectively.
The node update is
\begin{equation}
    \Delta_v
    =
    F_{\mathrm{out},v}\!\left(
        \sum_{i=1}^{K_v} g_{v,i}\odot z_{v,i}
    \right),
    \qquad
    h_v^{(l+1)}
    =
    h_{v,\mathrm{base}}^{(l+1)}
    +
    \Delta_v.
    \label{eq:graph_nsr_update}
\end{equation}
Here, \(h_{v,\mathrm{base}}^{(l+1)}\) is the update
produced by the original backbone on the current input.
The function \(F_{\mathrm{out},v}\) includes
neighborhood-dependent scaling and an output transformation
that maps the residual to the backbone output dimension,
with learnable parameters shared across nodes.

\subsection{NSR for Token Merging}
\label{sec:nsr_token}

Consider a merge group \(\mathcal{G}\) of \(K\) tokens
\(x_i\in\mathbb{R}^{d}\), with original merge weights satisfying
\(\sum_{i\in\mathcal{G}}t_i=1\).
Collecting these weights in \(t\), the current merge operator
is \(A_{\mathcal{G}}=t^\top\).

To match copy-based unmerging, we choose
\(U_{\mathcal{G}}=\mathbf{1}_K\), where \(\mathbf{1}_K\)
is the all-ones vector.
Since \(A_{\mathcal{G}}U_{\mathcal{G}}=1\),
Sec.~\ref{sec:nsr_complementarity} gives
\begin{equation}
    y_{\mathcal{G}}
    =
    \sum_{i\in\mathcal{G}}t_i x_i,
    \qquad
    z_i=x_i-y_{\mathcal{G}},
    \qquad
    \sum_{i\in\mathcal{G}}t_i z_i=0.
    \label{eq:token_nsr}
\end{equation}
NSR encodes each member residual \(z_i\) as
\(c_i\in\mathbb{R}^{r}\) following
Sec.~\ref{sec:nsr_task_adaptive} and uses the same codes
through feature and spatial residual paths.

\textbf{Feature residual.}
NSR aggregates the codes with the original merge weights
and adds the decoded correction to the merged representation:
\begin{equation}
    c_{\mathcal{G}}
    =
    \sum_{i\in\mathcal{G}}t_i c_i,
    \qquad
    y_{\mathcal{G}}^{\mathrm{NSR}}
    =
    y_{\mathcal{G}}
    +
    D_{\mathrm{feat}}(c_{\mathcal{G}}).
    \label{eq:token_feature_integration}
\end{equation}
Here, \(D_{\mathrm{feat}}:
\mathbb{R}^{r}\rightarrow\mathbb{R}^{d}\) decodes the group
code into token features.
The updated representation continues through the backbone
as a single token, allowing within-group variations to
affect subsequent transformations.

\textbf{Spatial residual.}
Copy-based unmerging cannot recover within-group distinctions.
NSR therefore retains member codes and their correspondence
to original patch positions.

For \(L\) merge events, let
\(C_l\in\mathbb{R}^{N_l\times r}\) collect the \(N_l\)
member codes retained at event \(l\).
The routing operation \(\Pi_l\) assigns these codes to
the \(N_0\) original patch positions using the recorded
correspondence.
All events share the same code dimension \(r\), giving
\begin{equation}
    S
    =
    \sum_{l=1}^{L}\Pi_l(C_l),
    \qquad
    X_{\mathrm{out}}
    =
    X_{\mathrm{base}}
    +
    D_{\mathrm{sp}}(S).
    \label{eq:token_spatial_routing}
\end{equation}
Here, \(S\in\mathbb{R}^{N_0\times r}\) accumulates spatial
codes, and \(X_{\mathrm{base}}\in\mathbb{R}^{N_0\times d}\)
is the representation after standard unmerging, including
preceding feature-residual updates.
At the end of the backbone, \(D_{\mathrm{sp}}\) decodes
the accumulated codes into \(d\)-dimensional position-specific
updates; the codes are not processed as an additional
token sequence by the Transformer.
Both paths preserve the original token-selection, matching,
and weight-computation rules, as well as the predefined
merge schedule.

% ============================================================

\section{Experiments}
\label{sec:experiments}

\subsection{Experimental Setup}
\label{sec:experimental-setup}

Semantic segmentation is evaluated on the official validation splits of
Pascal VOC 2012~\cite{everingham2012voc},
Cityscapes~\cite{cordts2016cityscapes}, and
ADE20K~\cite{zhou2017ade20k}, using mIoU.
All models use an ImageNet-pretrained~\cite{russakovsky2015imagenet}
DeiT-Tiny/16 backbone~\cite{touvron2021training} and a linear
segmentation head.
We compare ToMe~\cite{bolya2023tome},
PiToMe~\cite{tran2024pitome}, and MPM~\cite{rave2026mpm}
with their NSR-augmented counterparts, using the uncompressed
Full-ViT as a reference.

Graph experiments use GCN~\cite{kipf2016semi},
GraphSAGE~\cite{hamilton2017graphsage}, and GIN~\cite{xu2019powerful}.
Heterophilic node classification uses Roman-empire, Amazon-ratings,
Minesweeper, Tolokers, and Questions~\cite{platonov2023critical}.
The first two use accuracy, while the others use ROC-AUC.
Tree-NeighborsMatch~\cite{alon2021bottleneck} evaluates training fit
on binary trees of depths \(2\)--\(8\).
Molecular graph regression uses the \(12{,}000\)-graph ZINC
subset~\cite{dwivedi2023benchmarking}, evaluated by test MAE.
Heterophilic node classification and ZINC use four message-passing
layers with a hidden dimension of \(128\); the tree task at depth
\(d\) uses \(d+1\) message-passing layers with a hidden dimension
of \(32\).

Each Original--NSR comparison uses the same backbone configuration,
optimization settings, and training budget.
Segmentation models additionally share pretrained parameter
initialization, data ordering, and compression-control configurations.
Segmentation models are trained for \(100\) epochs and evaluated
using the final checkpoint; heterophilic node classification and
ZINC select checkpoints by validation performance.
Specific network configurations, preprocessing, and training
protocols are provided in Appendix~\ref{app:implementation}. The computation-reduction operating point (WP) is defined as the
ratio of Full-ViT GFLOPs to those of the original compressed baseline.
Computation is estimated analytically from the actual forward
structure.
Cityscapes computation is reported per input window, corresponding
to one window in full-image sliding-window inference.
Data processing, complete training configurations, and
computational-accounting conventions are provided in
Appendix~\ref{app:implementation}, while the construction of
compression configurations is described in
Appendix~\ref{app:token-operating-points}.

\subsection{Results}
\label{sec:results}

\subsubsection{Token Merging}
\label{sec:token-results}

Table~\ref{tab:token-main} reports semantic-segmentation performance
and the absolute gain over the corresponding compressed baseline,
\(\Delta_{\mathrm{NSR}}
=\mathrm{mIoU}_{\mathrm{NSR}}-\mathrm{mIoU}_{\mathrm{Base}}\).

\begin{table*}[t]
    \centering

    \caption{
        Semantic-segmentation results across merging methods and compression configurations.
        Bold marks the higher mIoU in each Original--NSR pair.
        $\Delta_{\mathrm{NSR}}$ is computed from unrounded mIoU.
    }
    \label{tab:token-main}

    % ==========================================================
    % Compact table settings
    % ==========================================================
    \begingroup

    % Font: smaller than \small, but still readable
    \fontsize{7.6}{7.8}\selectfont

    % Horizontal spacing
    \setlength{\tabcolsep}{3.0pt}

    % Main source of row-height compression
    \renewcommand{\arraystretch}{1.00}

    % IMPORTANT:
    % booktabs normally inserts substantial whitespace
    % above/below every rule. Compress it while keeping all rules.
    \setlength{\aboverulesep}{0.14ex}
    \setlength{\belowrulesep}{0.14ex}
    \setlength{\cmidrulesep}{0.10ex}

    \begin{tabular}{@{}llc cc cc c@{}}
        \toprule
        & & &
        \multicolumn{2}{c}{Original} &
        \multicolumn{2}{c}{+NSR} & \\[-0.4pt]

        \cmidrule(lr){4-5}
        \cmidrule(lr){6-7}

        Dataset & Method & WP
        & GFLOPs $\downarrow$ & mIoU $\uparrow$
        & GFLOPs $\downarrow$ & mIoU $\uparrow$
        & $\Delta_{\mathrm{NSR}} \uparrow$ \\
        \midrule

        % ======================================================
        % VOC
        % ======================================================
        \multirow{13}{*}{VOC}
        & Full-ViT & \(1.0\times\)
        & 1.254 & 64.85 & -- & -- & -- \\
        \cmidrule(lr){2-8}

        & \multirow{4}{*}{ToMe}
        & \(1.8\times\)
        & 0.698 & 62.81
        & 0.750 & \textbf{64.23}
        & +1.43 \\

        & & \(2.5\times\)
        & 0.492 & 44.58
        & 0.533 & \textbf{59.54}
        & +14.95 \\

        & & \(3.4\times\)
        & 0.373 & 18.31
        & 0.407 & \textbf{49.82}
        & +31.51 \\

        & & \(3.8\times\)
        & 0.333 & 16.98
        & 0.365 & \textbf{46.23}
        & +29.25 \\

        \cmidrule(lr){2-8}

        & \multirow{4}{*}{PiToMe}
        & \(1.8\times\)
        & 0.698 & 61.48
        & 0.736 & \textbf{62.44}
        & +0.96 \\

        & & \(2.5\times\)
        & 0.501 & 55.77
        & 0.536 & \textbf{59.19}
        & +3.42 \\

        & & \(3.4\times\)
        & 0.368 & 43.70
        & 0.400 & \textbf{54.02}
        & +10.31 \\

        & & \(3.8\times\)
        & 0.332 & 37.32
        & 0.362 & \textbf{50.66}
        & +13.34 \\

        \cmidrule(lr){2-8}

        & \multirow{4}{*}{MPM}
        & \(1.8\times\)
        & 0.699 & 61.99
        & 0.714 & \textbf{62.64}
        & +0.65 \\

        & & \(2.4\times\)
        & 0.512 & 58.87
        & 0.530 & \textbf{59.74}
        & +0.86 \\

        & & \(3.0\times\)
        & 0.429 & 54.93
        & 0.448 & \textbf{57.63}
        & +2.70 \\

        & & \(3.3\times\)
        & 0.376 & 51.34
        & 0.395 & \textbf{54.85}
        & +3.51 \\

        \midrule

        % ======================================================
        % Cityscapes
        % ======================================================
        \multirow{13}{*}{Cityscapes}
        & Full-ViT & \(1.0\times\)
        & 30.533 & 71.83 & -- & -- & -- \\
        \cmidrule(lr){2-8}

        & \multirow{4}{*}{ToMe}
        & \(1.8\times\)
        & 16.983 & \textbf{71.71}
        & 18.697 & 71.67
        & -0.04 \\

        & & \(2.5\times\)
        & 12.191 & 65.83
        & 13.633 & \textbf{70.59}
        & +4.76 \\

        & & \(3.4\times\)
        & 8.984 & 54.08
        & 10.155 & \textbf{69.41}
        & +15.33 \\

        & & \(3.8\times\)
        & 8.037 & 47.33
        & 9.125 & \textbf{68.16}
        & +20.83 \\

        \cmidrule(lr){2-8}

        & \multirow{4}{*}{PiToMe}
        & \(1.8\times\)
        & 16.910 & 69.70
        & 18.118 & \textbf{69.80}
        & +0.10 \\

        & & \(2.5\times\)
        & 12.213 & 69.01
        & 13.370 & \textbf{71.28}
        & +2.27 \\

        & & \(3.4\times\)
        & 8.962 & 65.77
        & 10.025 & \textbf{69.00}
        & +3.23 \\

        & & \(3.8\times\)
        & 8.041 & 64.21
        & 9.062 & \textbf{68.95}
        & +4.74 \\

        \cmidrule(lr){2-8}

        & \multirow{4}{*}{MPM}
        & \(1.8\times\)
        & 16.817 & 70.86
        & 17.098 & \textbf{71.34}
        & +0.48 \\

        & & \(2.6\times\)
        & 11.918 & 67.86
        & 12.328 & \textbf{69.75}
        & +1.89 \\

        & & \(3.0\times\)
        & 10.067 & 67.21
        & 10.557 & \textbf{68.40}
        & +1.19 \\

        & & \(3.4\times\)
        & 9.090 & 65.20
        & 9.624 & \textbf{68.02}
        & +2.82 \\

        \midrule

        % ======================================================
        % ADE20K
        % ======================================================
        \multirow{13}{*}{ADE20K}
        & Full-ViT & \(1.0\times\)
        & 10.463 & 37.60 & -- & -- & -- \\
        \cmidrule(lr){2-8}

        & \multirow{4}{*}{ToMe}
        & \(1.8\times\)
        & 5.817 & 36.60
        & 6.347 & \textbf{37.41}
        & +0.81 \\

        & & \(2.5\times\)
        & 4.202 & 31.44
        & 4.643 & \textbf{37.07}
        & +5.64 \\

        & & \(3.4\times\)
        & 3.085 & 17.67
        & 3.443 & \textbf{34.11}
        & +16.44 \\

        & & \(3.8\times\)
        & 2.750 & 12.44
        & 3.082 & \textbf{33.74}
        & +21.31 \\

        \cmidrule(lr){2-8}

        & \multirow{4}{*}{PiToMe}
        & \(1.8\times\)
        & 5.801 & 36.38
        & 6.181 & \textbf{37.94}
        & +1.56 \\

        & & \(2.5\times\)
        & 4.180 & 34.78
        & 4.540 & \textbf{35.95}
        & +1.17 \\

        & & \(3.4\times\)
        & 3.076 & 31.39
        & 3.404 & \textbf{35.65}
        & +4.26 \\

        & & \(3.8\times\)
        & 2.757 & 30.07
        & 3.071 & \textbf{34.17}
        & +4.11 \\

        \cmidrule(lr){2-8}

        & \multirow{4}{*}{MPM}
        & \(1.8\times\)
        & 5.804 & \textbf{37.49}
        & 5.977 & 37.48
        & -0.01 \\

        & & \(2.4\times\)
        & 4.443 & 36.59
        & 4.656 & \textbf{36.93}
        & +0.34 \\

        & & \(3.2\times\)
        & 3.304 & 33.57
        & 3.454 & \textbf{35.74}
        & +2.16 \\

        & & \(3.4\times\)
        & 3.055 & 32.69
        & 3.213 & \textbf{35.38}
        & +2.69 \\

        \bottomrule
    \end{tabular}

    \endgroup
\end{table*}

Across the 36 evaluated configurations, NSR achieves higher mIoU than
the corresponding compressed baseline in 34 cases. The two lower results
occur under mild compression and differ from the corresponding baselines
by no more than 0.04 points. The observed gains are generally larger
under stronger compression, with ToMe showing maximum gains of 31.51
and 20.83 points on Pascal VOC and Cityscapes, respectively. These
results indicate that NSR can provide substantial performance recovery
across different merging mechanisms, with the magnitude varying across
operators and compression configurations. NSR can also outperform a less compressed original model
while requiring less computation.
On ADE20K, ToMe with NSR at the \(3.4\times\) operating point
achieves \(34.11\) mIoU at \(3.443\) GFLOPs, exceeding the
\(31.44\) mIoU achieved by the original \(2.5\times\) ToMe
at \(4.202\) GFLOPs.

\subsubsection{Graph Aggregation}
\label{sec:graph-results}

\textbf{Tree-NeighborsMatch.}
This task requires the root to predict the value at the leaf
associated with a query key, and therefore depends on key--value
associations.
Consider the first local aggregation after the concatenated key
and value features undergo an initial linear embedding:
each member representation consists of a key encoding,
a value encoding, and a shared bias.
Fix two sibling leaves with distinct keys, swap two values with
distinct encodings in the message space, and keep all other
inputs and aggregation coefficients unchanged.
If the two leaves have equal coefficients, the aggregate at
their parent remains unchanged; however, when the query is fixed
to one of these keys, the correct answer changes.
This gives an explicit instance of task-relevant local
indistinguishability.

Following the task construction and training-fit evaluation protocol of \citet{alon2021bottleneck}, we use Tree-NeighborsMatch to diagnose whether GNNs can fit long-range key--value associations. Training accuracy is the benchmark's intended metric for detecting fitting failures as tree depth increases. We use $d+1$ message-passing layers on binary trees of depths $d=2$--$8$ and adopt random seed $11$ from their reference implementation. Each model is trained once at each depth, and we report the highest training accuracy attained during that run. As shown in Figure~3, the original backbones exhibit degraded fitting as depth increases, whereas all three NSR variants achieve 100\% training accuracy at $d=2$--$6$ and substantially outperform their counterparts at $d=7,8$. These results support the complementary pathway's effectiveness in improving training fit.

\begin{figure*}[t]
    \centering
    \includegraphics[width=0.7\linewidth]{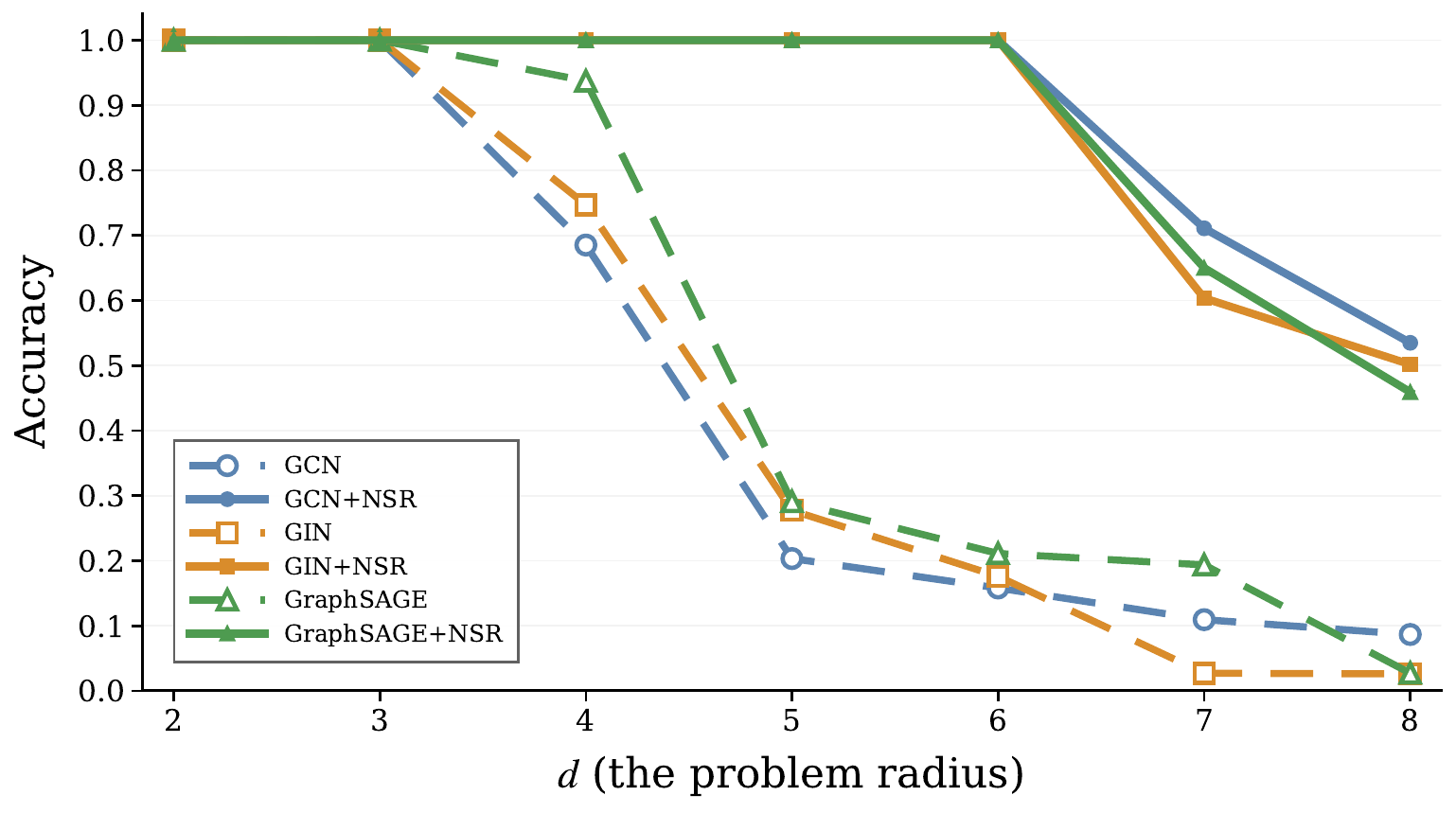}
    \caption{
        Training accuracy on Tree-NeighborsMatch across tree depths.
        Dashed and solid curves denote the original GCN, GIN,
        and GraphSAGE models and their NSR variants, respectively.
        Each point is the highest training accuracy attained
        in the corresponding run.
    }
    \label{fig:tree-depth}
\end{figure*}

\textbf{Heterophilic node classification.}
In Table~\ref{tab:heterophily-main}, NSR increases the reported
metric in \(13\) of the \(15\) backbone--dataset combinations.
The largest gains occur on Roman-empire, where accuracy improves
by \(10.85\), \(6.12\), and \(8.25\) percentage points for
GCN, GraphSAGE, and GIN, respectively.

\begin{table*}[t]
    \centering
    \caption{
        Node-classification results on the five heterophilic graph benchmarks.
        Values are the mean and sample standard deviation (\%) over the ten official splits.
        The higher score within each backbone pair is shown in bold.
    }
    \label{tab:heterophily-main}
    \scriptsize
    \setlength{\tabcolsep}{3.5pt}
    \renewcommand{\arraystretch}{1.0}
    \begin{tabular}{llccccc}
        \toprule
        Backbone & Variant
        & \shortstack{Roman-empire\\Acc. (\%) $\uparrow$}
        & \shortstack{Amazon-ratings\\Acc. (\%) $\uparrow$}
        & \shortstack{Minesweeper\\ROC-AUC (\%) $\uparrow$}
        & \shortstack{Tolokers\\ROC-AUC (\%) $\uparrow$}
        & \shortstack{Questions\\ROC-AUC (\%) $\uparrow$} \\
        \midrule

        \multirow{2}{*}{GCN}
        & Original
        & \(72.05 \pm 0.64\)
        & \(49.41 \pm 0.57\)
        & \(90.20 \pm 0.66\)
        & \(\mathbf{84.90 \pm 0.90}\)
        & \(75.32 \pm 1.35\) \\
        & +NSR
        & \(\mathbf{82.90 \pm 0.71}\)
        & \(\mathbf{49.67 \pm 0.57}\)
        & \(\mathbf{92.09 \pm 0.63}\)
        & \(84.36 \pm 0.73\)
        & \(\mathbf{78.25 \pm 1.21}\) \\
        \midrule

        \multirow{2}{*}{GIN}
        & Original
        & \(74.00 \pm 0.79\)
        & \(49.56 \pm 0.55\)
        & \(86.90 \pm 0.57\)
        & \(83.06 \pm 0.88\)
        & \(74.81 \pm 1.57\) \\
        & +NSR
        & \(\mathbf{82.25 \pm 0.65}\)
        & \(\mathbf{50.47 \pm 0.57}\)
        & \(\mathbf{89.41 \pm 1.09}\)
        & \(\mathbf{83.18 \pm 0.87}\)
        & \(\mathbf{75.06 \pm 1.40}\) \\
        \midrule

        \multirow{2}{*}{GraphSAGE}
        & Original
        & \(81.59 \pm 0.55\)
        & \(52.33 \pm 0.37\)
        & \(\mathbf{93.37 \pm 0.40}\)
        & \(83.28 \pm 0.56\)
        & \(74.82 \pm 0.91\) \\
        & +NSR
        & \(\mathbf{87.71 \pm 0.61}\)
        & \(\mathbf{52.72 \pm 0.75}\)
        & \(92.75 \pm 0.48\)
        & \(\mathbf{84.50 \pm 0.63}\)
        & \(\mathbf{75.07 \pm 2.16}\) \\

        \bottomrule
    \end{tabular}
\end{table*}

Platonov et al.~\cite{platonov2023critical} identify Roman-empire
as the only one of these five datasets with Label Informativeness
(LI) appreciably above zero.
LI quantifies how much information a neighbor's label provides
about a node's label~\cite{platonov2023characterizing}.
The larger NSR gains on Roman-empire are consistent with this
label structure: preserving fine-grained distinctions compressed
by aggregation may be more valuable when neighborhoods contain
stronger task-predictive information.
Performance changes vary across datasets and backbones, consistent with the premise that residual utility depends on task relevance and the model's ability to exploit it.
LI is defined from graph labels and serves as auxiliary context
here, rather than a direct measurement of the hidden-feature
residuals processed by NSR.

\textbf{Label-based null-space diagnostic.}
We further examine the relationship between NSR gains and label
structure in the aggregation null space on Roman-empire.
We project member one-hot label encodings onto the null space
of the corresponding local aggregation operator, compute
normalized projection energies, and average them over the
target node's computation tree with weights given by local
member counts.
This yields a node-level diagnostic score.
The score is computed after training and is not used for
training, hyperparameter selection, or model selection.
Its full definition is provided in Appendix~\ref{app:trnsil}.

Figure~\ref{fig:trnsil-accuracy} shows that the original backbones'
mean accuracy generally decreases across higher score bins,
while the NSR variants remain relatively stable, producing
an overall widening accuracy gap.
This association is consistent with the motivation of learning
to exploit aggregation-invisible distinctions.
The corresponding gain curves are provided in
Appendix~\ref{app:trnsil}.

\begin{figure*}[t]
    \centering
    \includegraphics[width=0.7\linewidth]{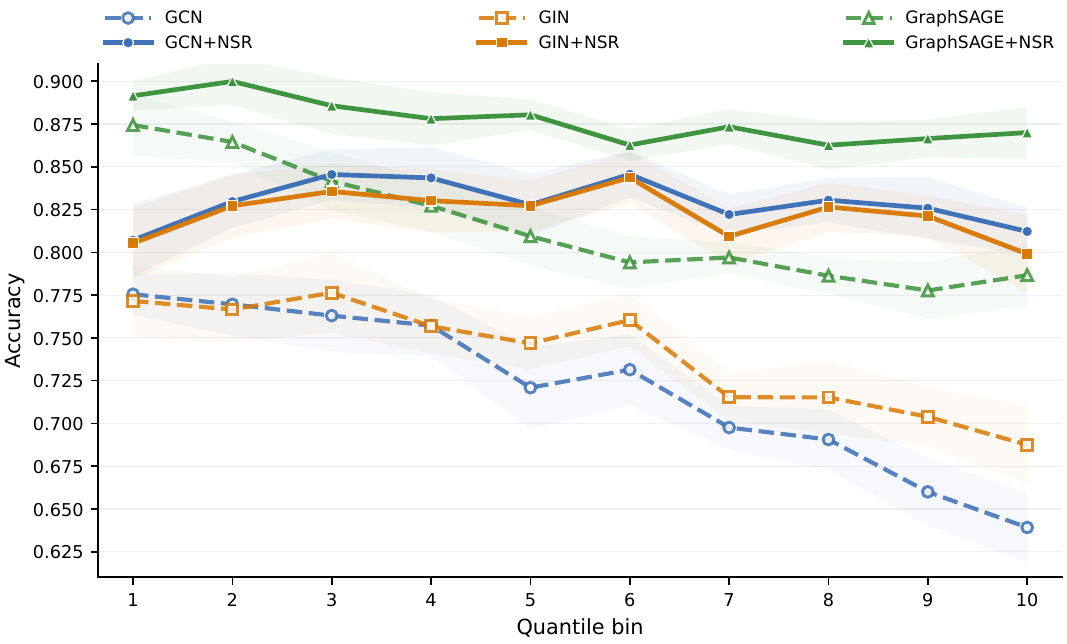}
    \caption{
        Within each official split, test nodes are divided into
ten approximately equal-frequency bins, with the horizontal axis
ordering the bins from low to high scores.
Lines show the mean accuracy of the corresponding bin across the
ten official splits, and shaded bands of the same color indicate
the mean $\pm$ one sample standard deviation.
    }
    \label{fig:trnsil-accuracy}
\end{figure*}

\textbf{ZINC graph regression.}
To test whether NSR can serve as a more general mechanism
for compensating information discarded by aggregation,
we evaluate graph-level molecular property regression on ZINC.
This setting differs from the preceding experiments in data
semantics, prediction granularity, and learning objective:
the model predicts a continuous molecular property through
repeated local aggregation and graph-level readout.
To keep aggregation backbones consistent across graph tasks,
we retain the original node-feature-based formulations of
GCN, GraphSAGE, and GIN, using only atom types and molecular
connectivity without bond-type attributes.

Table~\ref{tab:zinc-results} shows that NSR lowers the mean
test MAE of all three backbones, with absolute reductions
of \(0.1785\), \(0.0360\), and \(0.0542\) for GCN, GIN,
and GraphSAGE, respectively.
This supports the effectiveness of the complete NSR pathway
beyond matching and node classification, extending to
graph-level regression.

\begin{table}[t]
    \centering
    \caption{
        Test MAE on the ZINC subset.
        Values are the mean and standard deviation over
        ten random seeds; lower values are better.
    }
    \label{tab:zinc-results}
    \small
    \setlength{\tabcolsep}{7pt}
    \renewcommand{\arraystretch}{1.05}
    \begin{tabular}{lcc}
        \toprule
        Backbone & Original & +NSR \\
        \midrule
        GCN
        & \(0.4732 \pm 0.0067\)
        & \(\mathbf{0.2947 \pm 0.0050}\) \\
        GIN
        & \(0.3451 \pm 0.0093\)
        & \(\mathbf{0.3091 \pm 0.0093}\) \\
        GraphSAGE
        & \(0.4373 \pm 0.0086\)
        & \(\mathbf{0.3831 \pm 0.0125}\) \\
        \bottomrule
    \end{tabular}
\end{table}

These graph experiments provide complementary evidence:
Tree-NeighborsMatch examines fitting under explicit local
mismatch; the heterophily experiments extend evaluation to
real graph data, with the label-based diagnostic examining
how predictive gains relate to label structure; and ZINC
extends validation to graph-level regression.
Together, they support using null-space residuals to
complement the original aggregation.

Overall, experiments across token merging and graph aggregation support the practical value of NSR in exploiting operator-invisible member distinctions to improve task performance while preserving the original mapping rules. Comparisons and ablations in both settings provide further evidence for the framework's design (Appendix~\ref{app:additional}).
% ============================================================
\section{Conclusion}

Motivated by the potential mismatch between operator-induced input equivalences and downstream task requirements, we propose Task-Relevant Null-Space Residuals (NSR), a general residual framework for non-injective linear mappings. NSR combines operator-defined null-space residual extraction, member-level encoding and gating, and application-specific integration into a task-supervised complementary pathway while preserving the original aggregation or merging rules. The null space of the currently realized linear operator defines candidate complementary information, while downstream supervision guides its use. Experiments in graph aggregation and token merging demonstrate improvements across multiple models and configurations. These results support a practical design principle: for many-to-one neural computation, explicitly providing access to operator-invisible variations has practical value.

% ============================================================
\bibliographystyle{unsrtnat}
\bibliography{references}

% ============================================================
\appendix

\section{Implementation and Experimental Details}
\label{app:implementation}

This appendix describes the implementation of NSR for token merging
and graph aggregation and supplements the experimental settings in
the main text.
In both applications, NSR constructs member residuals using the
members and weights of the current aggregation operation and
transforms them into task-adaptive updates through a learnable branch.

\subsection{Datasets, Baselines, and Experimental Protocols}
\label{app:training-evaluation}

This section provides the datasets, comparison settings, training and
evaluation protocols, and computational-accounting conventions for the
experiments in the main text. The implementations of NSR for token merging
and graph aggregation are described in
Appendices~\ref{app:token-implementation}
and~\ref{app:graph-implementation}, respectively.

\paragraph{Semantic-segmentation data and preprocessing.}
Pascal VOC 2012~\cite{everingham2012voc} contains \(20\) foreground
classes and one background class, with \(1{,}464\) training images and
\(1{,}449\) validation images.
Cityscapes~\cite{cordts2016cityscapes} uses the fine-annotation split,
containing \(2{,}975\) training images and \(500\) validation images with
\(19\) evaluated semantic classes.
ADE20K~\cite{zhou2017ade20k} contains \(20{,}210\) training images and
\(2{,}000\) validation images covering \(150\) semantic classes.
All experiments use the official training and validation splits.

VOC 2012 retains its original class encoding.
ADE20K maps valid class labels to \(0\)--\(149\), and Cityscapes maps the
original annotations to its \(19\) training classes.
Invalid labels are assigned the ignore index \(255\) and excluded from
both the loss and evaluation metrics.
All images are normalized using the ImageNet channel means and
standard deviations.

For VOC and ADE20K, training augmentation consists of
aspect-ratio-preserving random resizing, random cropping, and horizontal
flipping with probability \(0.5\).
The resizing scale is sampled from \(0.5\)--\(2.0\), and the crop sizes
are \(224\times224\) and \(512\times512\), respectively.
Validation resizes the shorter image side to \(224\) and \(512\),
respectively, followed by a center crop of the corresponding size.
Cityscapes uses random resizing, photometric distortion,
\(512\times1024\) random cropping, and horizontal flipping.
Validation performs sliding-window inference on full images, using a
\(512\times1024\) window and a \(341\times682\) stride.
Class logits are averaged in overlapping regions.
Geometric transformations are applied jointly to images and labels,
with nearest-neighbor interpolation used for labels.

\paragraph{Compared methods and controlled comparisons.}
The segmentation experiments compare ToMe, PiToMe, and MPM with their
respective NSR-augmented versions and include the uncompressed Full-ViT
as a reference.
All models use the same DeiT-Tiny/16 segmentation architecture.
ToMe and PiToMe retain their token-size-based merging weights, while
MPM retains mutual-nearest-neighbor pairing and arithmetic-mean merging.
The merge locations and within-group weights are summarized in
Table~\ref{tab:token-nsr-integration}.

Graph experiments use GCN, GraphSAGE, and GIN, with NSR branches added
to their neighborhood aggregation layers.
Residual construction follows the member sets and aggregation
coefficients used by the corresponding backbone, as specified in
Table~\ref{tab:graph-nsr-integration}.
Each pair shares the same backbone configuration, task inputs, loss
function, and training protocol.

Before segmentation training, shared backbone and segmentation-head
parameters are copied across the compared models.
Random seeds are reset and data loaders are reconstructed for each
model, ensuring consistent data ordering and training budgets on the
same dataset.
Each Original--NSR pair uses the same compression-control configuration;
adding NSR preserves the original token-selection, matching, and
merging rules.

\paragraph{Semantic-segmentation training and evaluation.}
All segmentation models use an ImageNet-pretrained DeiT-Tiny/16 backbone.
After the representations are restored to the original patch grid,
LayerNorm and a linear classification layer produce class logits,
which are bilinearly upsampled to the input image resolution.
The backbone, segmentation head, and NSR parameters are trained jointly
under downstream segmentation supervision.

Training uses pixel-wise cross-entropy and
AdamW~\cite{loshchilov2019decoupled}, with an initial
learning rate of \(10^{-4}\) and weight decay of \(10^{-4}\).
The learning rate follows polynomial decay with power \(0.9\),
updated at each training iteration.
Configurations with warm-up first increase the learning rate linearly
to its initial value.
Training runs for \(100\) epochs on all three datasets.
CUDA training uses automatic mixed precision and gradient scaling.

For each semantic-segmentation configuration, we report
validation mIoU from the final checkpoint of a single training run.
The mIoU is computed from the confusion matrix accumulated
over the entire validation set and averaged over classes
with nonzero union.
The mIoU is computed from the confusion matrix accumulated over the
entire validation set and averaged over classes with nonzero union.
The main results and GFLOPs are reported in Table~\ref{tab:token-main},
with additional performance and parameter results in
Table~\ref{tab:token-additional}.

\paragraph{Heterophilic node classification.}
We use Roman-empire, Amazon-ratings, Minesweeper, Tolokers, and
Questions~\cite{platonov2023critical}.
Roman-empire is a word-dependency graph derived from an English Wikipedia
article, Amazon-ratings is a product co-purchasing network, and
Minesweeper is a synthetic graph based on the corresponding game.
Tolokers and Questions are interaction networks constructed from a
crowdsourcing platform and a question-answering service, respectively.
Dataset statistics are provided in
Table~\ref{tab:heterophily-statistics}.
We follow the ten fixed splits supplied with the benchmark.
Roman-empire and Amazon-ratings use classification accuracy, while the
remaining three datasets use ROC-AUC.

The models use four message-passing layers, a hidden dimension of
\(128\), and dropout of \(0.5\).
Training uses the full graph, with cross-entropy loss computed only on
training nodes.
Adam~\cite{kingma2015adam} uses an initial learning rate of
\(10^{-3}\) and weight decay of \(5\times10^{-4}\),
with the gradient norm clipped to \(1\).
Each official split uses random seed \(42\), and evaluation is performed
every ten epochs.

The learning-rate scheduler monitors the validation metric, with a
reduction factor of \(0.5\), a patience of \(50\) evaluations, and a
minimum learning rate of \(10^{-6}\).
Training runs for at most \(3{,}000\) epochs and stops after \(100\)
validation evaluations without improvement.
For each split, we report the test score of the checkpoint with the
best validation metric.
Final results are the mean and sample standard deviation over the
ten splits.
Results are reported in Table~\ref{tab:heterophily-main}.

\paragraph{Tree-NeighborsMatch.}
We follow the task construction of Alon and Yahav~\cite{alon2021bottleneck}
and use training accuracy to examine fitting ability on this controlled
matching task.
The root must predict the value stored at the leaf associated with its
query key.
Each leaf contains a key--value pair, and node inputs concatenate
one-hot encodings of keys and values.
We use complete binary trees of depths \(2\)--\(8\), with edges directed
from children to parents and self-loops included.
Examples are divided into \(80\%\) training and \(20\%\) test sets using
label-stratified splitting, without a separate validation set.

For a tree of depth \(d\), the model uses \(d+1\) message-passing layers
and a hidden dimension of \(32\).
Table~\ref{tab:tree-training-config} lists the generated sample counts,
message-passing depths, and nominal effective batch sizes.

\begin{table}[t]
    \centering
    \caption{
        Tree-NeighborsMatch configurations.
        Effective batches larger than \(256\) are implemented using
        gradient accumulation.
    }
    \label{tab:tree-training-config}
    \small
    \setlength{\tabcolsep}{6pt}
    \renewcommand{\arraystretch}{1.08}
    \begin{tabular}{@{}crcc@{}}
        \toprule
        Tree depth
        & \shortstack{Generated\\samples}
        & \shortstack{Message-passing\\layers}
        & \shortstack{Nominal effective\\batch size} \\
        \midrule
        2 & 96     & 3 & 64    \\
        3 & 8,000  & 4 & 64    \\
        4 & 16,000 & 5 & 1,024 \\
        5 & 32,000 & 6 & 1,024 \\
        6 & 32,000 & 7 & 1,024 \\
        7 & 32,000 & 8 & 2,048 \\
        8 & 32,000 & 9 & 2,048 \\
        \bottomrule
    \end{tabular}
\end{table}

Training uses cross-entropy and Adam with an initial learning rate of
\(10^{-3}\), without dropout or weight decay, for at most \(50{,}000\)
epochs.
Learning-rate scheduling and early stopping both monitor training
accuracy, with patience values of \(1{,}000\) and \(2{,}000\) epochs,
respectively.
All models use random seed \(11\) and are trained once at each tree depth.
We report the highest training accuracy attained during each run.
Figure~\ref{fig:tree-depth} shows training fit across tree depths.

\paragraph{ZINC.}
We use the \(12{,}000\)-graph ZINC subset~\cite{dwivedi2023benchmarking},
with fixed training, validation, and test splits containing
\(10{,}000\), \(1{,}000\), and \(1{,}000\) graphs, respectively.
The task is molecular property regression, evaluated by MAE, for which
lower values indicate better performance.
All models use atom types and graph connectivity without bond-type
features.

Atom types are mapped to \(128\)-dimensional learnable embeddings.
The models use four message-passing layers, a hidden dimension of
\(128\), and dropout of \(0.5\).
Training minimizes the \(L_1\) loss using Adam, with an initial learning
rate of \(10^{-3}\), weight decay of \(10^{-5}\), and a batch size of
\(128\).
The gradient norm is clipped to \(1\), and training runs for at most
\(400\) epochs.

Validation MAE determines both learning-rate scheduling and checkpoint
selection.
The scheduler uses a patience of \(20\) epochs, a reduction factor of
\(0.5\), and a minimum learning rate of \(10^{-5}\).
Early stopping uses a patience of \(100\) epochs.
After training, the checkpoint with the lowest validation MAE is
restored and evaluated once on the test set.
Table~\ref{tab:zinc-results} reports the mean and standard deviation of
test MAE over random seeds \(0\)--\(9\).

\paragraph{Operating points and configuration consistency.}
Token-merging experiments cover operating points from mild to
aggressive compression.
The computation-reduction factor is defined as
\(C=F_{\mathrm{Full}}/F_{\mathrm{Base}}\), where
\(F_{\mathrm{Full}}\) and \(F_{\mathrm{Base}}\) denote the GFLOPs of
the uncompressed Full-ViT and the original compressed baseline,
respectively.
Each method adjusts only its native compression controls to approach
the target computation budget.

ToMe varies the planned number of tokens merged per layer, PiToMe
varies the token-retention ratio, and MPM varies the Transformer blocks
at which merging is applied.
The original matching rules, weight computations, and merging
mechanisms remain unchanged.
Because MPM pairing depends on input content and model representations,
the compression factors achieved during calibration and final
evaluation may differ.
We retain the selected insertion schedule and report the computation
corresponding to final evaluation.

Each Original--NSR pair uses the same compression-control configuration:
ToMe shares the planned merge count and layer-wise merge configuration;
PiToMe shares the retention ratio, margin schedule, and matching mode
schedule; and MPM shares the insertion-block schedule.
NSR does not increase compression to compensate for its additional
computational overhead, and its total GFLOPs are reported separately.
The configurations for the three methods are provided in
Tables~\ref{tab:tome-compression-configurations},
\ref{tab:pitome-compression-configurations},
and~\ref{tab:mpm-compression-configurations}, respectively.
The complete configuration-search procedure is described in
Appendix~\ref{app:token-operating-points}.

\paragraph{Computational accounting.}
Computation is estimated analytically from the model's forward
structure, counting one multiply--accumulate as one FLOP.
The calculation includes patch embedding, attention, MLPs, the main
matching and merging operations, and the linear segmentation head.
For NSR models, it additionally includes residual encoding, gating,
feature decoding, spatial writeback, and spatial decoding.
Normalization, activation, indexing operations, and final upsampling
are excluded.
Parameter counts include all trainable parameters of the complete model.

For ToMe and PiToMe, layer-wise token counts are determined by the
compression configurations.
For MPM, computation is calculated from each sample's valid pair counts
and then averaged across samples, excluding the additional cost of
batch padding.
Cityscapes GFLOPs correspond to a single \(512\times1024\) input window,
and MPM pair statistics are collected from center windows of validation
images.
Total computation for full-image sliding-window inference also depends
on the number of windows.

\paragraph{Computing environment.}
Experiments were conducted on four compute instances, each equipped
with \(16\) CPU cores, \(128\,\mathrm{GB}\) of RAM, and one NVIDIA A100 GPU.
The implementation is based on PyTorch~\cite{paszke2019pytorch}.
Graph experiments use PyTorch Geometric~\cite{fey2019geometric} and torch-scatter, while vision
experiments use timm~\cite{wightman2019timm} and torchvision.

\subsection{Token-Merging Implementation}
\label{app:token-implementation}

All semantic-segmentation models use an ImageNet-pretrained
DeiT-Tiny/16 backbone.
For different input resolutions, the pretrained patch positional
embeddings are resized to the target grid using bicubic interpolation,
while the positional embedding of the classification token is retained.
After the backbone representations are restored to the original patch
grid, LayerNorm and a linear classification layer produce class logits,
which are bilinearly upsampled to the input image resolution.

ToMe, PiToMe, and MPM use the same NSR branch architecture.
For each method, residuals are constructed from the member groups and
weights produced by its merging operation.
Table~\ref{tab:token-nsr-integration} summarizes the merge locations
and within-group weights.

\begin{table}[t]
    \centering
    \caption{
        Integration of NSR with the three token-merging methods.
        Each method retains its merge location and weighting rule.
        Token size denotes the number of original patches represented
        by a token.
    }
    \label{tab:token-nsr-integration}
    \small
    \setlength{\tabcolsep}{5pt}
    \renewcommand{\arraystretch}{1.12}
    \begin{tabularx}{\linewidth}{
        @{}l
        >{\raggedright\arraybackslash}X
        >{\raggedright\arraybackslash}X@{}
    }
        \toprule
        Method & Merge location & Within-group weights \\
        \midrule
        ToMe
        & After the attention residual update and before the MLP
        & Normalized token-size weights \\
        PiToMe
        & After the attention residual update and before the MLP
        & Normalized token-size weights \\
        MPM
        & Before a designated Transformer block
        & Equal weights of \(1/2\) within each pair \\
        \bottomrule
    \end{tabularx}
\end{table}

ToMe and PiToMe use proportional attention based on token size.
MPM uses mutual-nearest-neighbor pairing and arithmetic-mean merging.
The classification token participates in backbone attention but is
excluded from matching and merging.
The operating-point configurations of all three methods are provided
in Appendix~\ref{app:token-operating-points}.

For a merge group \(\mathcal{G}\), NSR computes the merged
representation using the original within-group weights and extracts
the residual of each member relative to that representation.
For ToMe and PiToMe, all source tokens assigned to the same destination
are organized into a complete merge group.
Each valid MPM group contains two members.
Valid-member masks exclude padded entries introduced by batching,
and singleton groups produce no residual update.

The residual encoder consists of RMSNorm~\cite{zhang2019rmsnorm} followed by a bias-free
linear layer that maps the \(d\)-dimensional residual to a code of
dimension \(d_c=64\).
The gating network receives the concatenation of the normalized
member residual, member representation, and group representation.
It uses two linear layers with dimensions
\(3d\rightarrow32\rightarrow64\), a GELU activation between them,
and a sigmoid output to generate member-dependent, channel-wise
weights.
The member and group representations condition the gate, while the
information encoded by the residual encoder comes from the member
residual.

The same member codes support the feature and spatial residual paths.
The feature path aggregates the member codes using the original merge
weights and applies a bias-free \(64\rightarrow d\) linear decoder.
The resulting update is added to the current merged token, which
continues through the backbone.
Different merge locations have separate encoders, gating networks,
and feature decoders.

The spatial path maintains the correspondence between current tokens
and original patch positions.
At each merge event, a member code is written to the original
positions covered by that member and accumulated in a fixed
\(N_0\times64\) code tensor.
After the backbone computation, the merge events are reversed to
restore the original patch ordering.
A single shared bias-free linear decoder maps the accumulated codes
to position-specific updates, which are added to the restored patch
representations.
The spatial code tensor is initialized separately for each forward
pass and is not processed as an additional token sequence by the
Transformer.

Both the feature decoders and the final spatial decoder are
zero-initialized.
Consequently, with identical shared parameters, the NSR-augmented
model initially preserves the output of its corresponding merging
baseline.
The original and NSR-augmented models use the same matching rules,
weight-computation procedures, and predefined compression
configurations.
NSR constructs its residuals from the groups realized at each merge
event.

The feature-only and spatial-only ablations retain the corresponding
update path.
The post-merge-only control generates a feature update from the
merged group representation without using member residuals or
spatial correspondences.
The associated results are reported in
Appendix~\ref{app:additional}.

\subsection{Graph-Aggregation Implementation}
\label{app:graph-implementation}

The graph experiments use GCN, GIN, and GraphSAGE as backbones.
All models first map node inputs into the hidden representation
space and then apply multiple message-passing layers.
Each message-passing layer has a separate NSR branch, with no
parameter sharing across layers.

The graph branch applies its own bias-free \(d\rightarrow d\)
linear layer to the node representations participating in the
current aggregation and stacks the resulting messages row-wise
to form the member-message matrix \(M_v\) in
Section~\ref{sec:nsr_graph}.
The parameters of this linear layer are shared across members.
Using the member set and aggregation coefficients \(t_v\) of
the corresponding backbone, the branch applies the null-space
projection in Eq.~(\ref{eq:graph_nsr}) to \(M_v\) to extract
member-level residuals.
The projection properties are derived in
Appendix~\ref{app:nsr_theory}.
The implementation uses scatter operations grouped by destination
node and does not explicitly materialize local projection matrices.

Table~\ref{tab:graph-nsr-integration} specifies the members and
coefficients used by each backbone.

\begin{table}[t]
    \centering
    \caption{Member sets and aggregation coefficients used by the NSR graph branches.}
    \label{tab:graph-nsr-integration}
    \small
    \setlength{\tabcolsep}{5pt}
    \renewcommand{\arraystretch}{1.12}
    \begin{tabularx}{\linewidth}{
        @{}l
        >{\raggedright\arraybackslash}X
        >{\raggedright\arraybackslash}X@{}
    }
        \toprule
        Backbone & Members & Aggregation coefficients \\
        \midrule
        GCN
        & Incoming messages in the normalized backbone graph,
          including self-loops
        & Degree-normalized GCN coefficients \\
        GraphSAGE
        & Neighbor messages specified by the input edge list
        & Neighborhood-mean weights \\
        GIN
        & Neighbor messages and a separate root message
        & \(1\) for neighbors and \(1+\epsilon\) for the root \\
        \bottomrule
    \end{tabularx}
\end{table}

Each branch follows the incoming-edge and self-loop handling of its
corresponding backbone.
The separate root branch of GraphSAGE is retained in the backbone.
For GIN, the residual construction includes both the neighbor
aggregation terms and the separate root term.
The evaluated GIN configuration uses a fixed \(\epsilon=0\).

Member residuals are modulated by channel-wise gates.
The gating network concatenates the source-node representation,
destination-node representation, and member residual.
It consists of two linear layers with dimensions
\(3d\rightarrow d\rightarrow d\), a GELU activation, and a sigmoid
output.
The bias of its final linear layer is initialized to \(-2\).
The gated residuals are summed by destination node and scaled by
the inverse square root of the number of participating members.
The resulting node-level representation is passed through LayerNorm,
Dropout, and a \(d\rightarrow d\rightarrow d\) output network with
a GELU activation between its linear layers to produce the correction.

For GCN and GraphSAGE, the baseline layer applies graph convolution,
LayerNorm, GELU, and Dropout, followed by a residual addition to the
layer input.
The GIN convolution contains two successive
Linear--BatchNorm--ReLU stages.
Its output is passed through Dropout and added to the layer input.
In all cases, the NSR correction is added to the corresponding
baseline layer update, retaining the original backbone structure.

Node classification applies a linear classifier to the final node
representations.
Tree-NeighborsMatch reads out only the root representation.
ZINC uses global mean pooling to obtain a graph representation,
followed by a linear regression head.

% ============================================================
\section{Supplementary Derivations for NSR}
\label{app:nsr_theory}

This section explains why the residual operator in the main text
is a null-space projection and verifies that the lifts used for
graph aggregation and token merging satisfy the conditions of
the general construction.
All derivations concern the linear operator realized in a single
computation event.
When group assignments or weights depend on the input, they are
held fixed at the values realized for that event.

\subsection{Null-Space Projection Properties}
\label{app:null_projection}

Following the notation in Sec.~\ref{sec:nsr_complementarity},
let \(A\in\mathbb{R}^{k\times n}\) denote the current linear
operator, where \(n\) and \(k\) are the numbers of input and
output members, respectively.
Let \(U\in\mathbb{R}^{n\times k}\) be a linear lift from the
output representation space to the input member space,
satisfying \(AUA=A\)~\cite{benisrael2003generalized}.
Define the residual operator as \(P=I_n-UA\), where \(I_n\)
is the \(n\times n\) identity matrix.

We verify that \(P\) maps inputs into the null space of \(A\)
and leaves vectors in that null space unchanged.
Write
\(\ker(A)=\{q\in\mathbb{R}^{n}:Aq=0\}\).
Using \(AUA=A\), we obtain
\begin{equation}
    \begin{aligned}
        AP
        &=A-AUA=0,
        \\
        P^2
        &=I_n-2UA+U(AUA)
        =I_n-UA
        =P.
    \end{aligned}
    \label{eq:app_projection_proof}
\end{equation}

The identity \(AP=0\) implies that the range of \(P\) is
contained in \(\ker(A)\).
Conversely, for any \(q\in\ker(A)\), we have
\(Pq=q-UAq=q\), so \(\ker(A)\) is also contained in the
range of \(P\).
Together with \(P^2=P\), these relations establish that \(P\)
is a projection onto \(\ker(A)\), though not necessarily an
orthogonal projection.
Thus, the residual operator in the main text projects each
feature channel onto directions invisible to the original
operator and acts as the identity on those directions.

\subsection{Specialization to Graph Aggregation and Token Merging}
\label{app:nsr_specializations}

The two applications use different lifts: graph aggregation
lifts along the aggregation coefficient vector, whereas token
merging copies the merged representation back to the member
positions.
We verify below that both choices satisfy \(AUA=A\), allowing
direct application of
Proposition~\ref{prop:nsr_complementarity}.

\paragraph{Graph aggregation.}
\label{app:graph_specialization}

Fix a local aggregation at node \(v\).
Let \(K_v\geq1\) be the number of participating members and
let \(t_v\in\mathbb{R}^{K_v}\) be the nonzero vector of
aggregation coefficients.
This operation reduces the member messages to a single output
through a weighted sum, with operator
\(A_v=t_v^\top\in\mathbb{R}^{1\times K_v}\).
Following the construction in Sec.~\ref{sec:nsr_graph}, choose
the lift
\(U_v=t_v/(t_v^\top t_v)\in\mathbb{R}^{K_v\times1}\).

Since \(t_v^\top t_v>0\), this lift is well defined, and
\begin{equation}
    A_vU_v
    =
    \frac{t_v^\top t_v}{t_v^\top t_v}
    =1,
    \qquad
    A_vU_vA_v=A_v.
    \label{eq:app_graph_lift_condition}
\end{equation}

This choice therefore satisfies the condition of the general
construction, and the graph residual in the main text is
obtained by a null-space projection for the current local
aggregation operator.
The member set and coefficients are those of the aggregation
actually performed by the corresponding backbone.

\paragraph{Token merging.}
\label{app:token_specialization}

Consider a realized merge group \(\mathcal{G}\) containing
\(K\geq1\) tokens.
Let \(t\in\mathbb{R}^{K}\) be the merge weights, satisfying
\(t^\top\mathbf{1}_K=1\), where \(\mathbf{1}_K\) is the
\(K\)-dimensional all-ones column vector.
The merge operator is
\(A_{\mathcal{G}}=t^\top\in\mathbb{R}^{1\times K}\).
Following the copy-based unmerging used in
Sec.~\ref{sec:nsr_token}, choose
\(U_{\mathcal{G}}=\mathbf{1}_K\in\mathbb{R}^{K\times1}\),
which copies the same merged representation to every member
position in the group.

The weight normalization gives
\begin{equation}
    A_{\mathcal{G}}U_{\mathcal{G}}
    =
    t^\top\mathbf{1}_K
    =1,
    \qquad
    A_{\mathcal{G}}U_{\mathcal{G}}A_{\mathcal{G}}
    =A_{\mathcal{G}}.
    \label{eq:app_token_lift_condition}
\end{equation}

Thus, the copy-based lift also satisfies the condition of the
general construction, and the difference between the member
representations and the copied merged representation is the
corresponding null-space residual.
This conclusion depends only on the normalization of the
weights within the current group and does not require uniform
weights.

These verifications connect both residual constructions to the
unified formulation in the main text.
Their null-space properties and complementarity follow directly
from Proposition~\ref{prop:nsr_complementarity}.

% ============================================================
\section{Numerical Verification of NSR for Token Merging}
\label{app:nsr-numerical}

To verify that the implemented NSR construction is consistent with its
theoretical null-space definition, we examine both the null-space constraint
and the magnitude of the resulting residuals. For each realized token-merging
operation, let \(A\) denote the corresponding linear aggregation operator,
\(X\) the input token representations, and \(Z\) the residual defined in
Eq.~\eqref{eq:token_nsr}. We measure
\begin{equation}
    \epsilon_{\mathrm{null}}
    =
    \|AZ\|_{\infty},
    \qquad
    \rho_Z
    =
    \frac{\|Z\|_F}{\|X\|_F}.
    \label{eq:nsr-numerical-metrics}
\end{equation}
Here, \(\epsilon_{\mathrm{null}}\) measures the numerical violation of the
null-space constraint \(AZ=0\), while \(\rho_Z\) characterizes the residual
magnitude relative to the input representation.

\begin{figure*}[t]
    \centering
    \begin{minipage}[t]{0.32\textwidth}
        \centering
        \includegraphics[width=\linewidth]{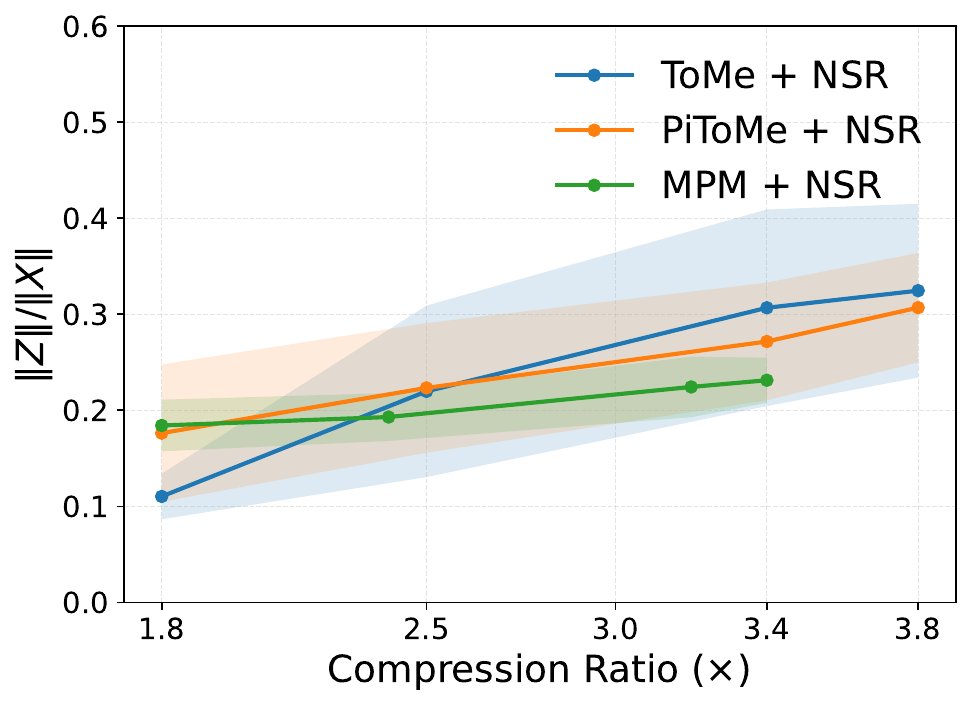}
        \vspace{-0.8em}
        \centerline{\small (a) ADE20K}
    \end{minipage}
    \hfill
    \begin{minipage}[t]{0.32\textwidth}
        \centering
        \includegraphics[width=\linewidth]{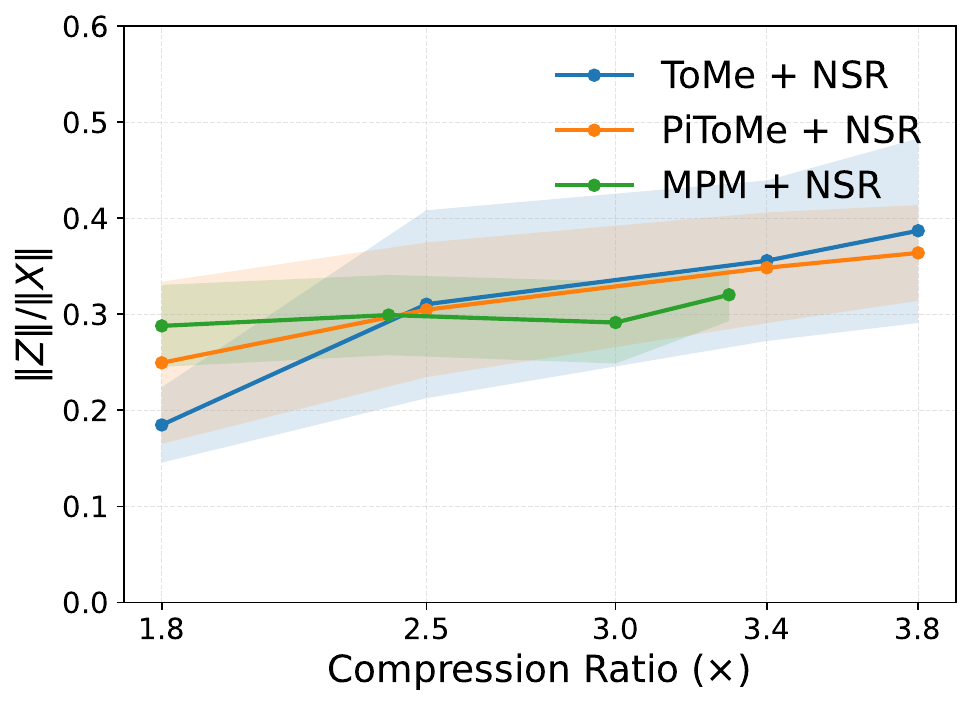}
        \vspace{-0.8em}
        \centerline{\small (b) Pascal VOC}
    \end{minipage}
    \hfill
    \begin{minipage}[t]{0.32\textwidth}
        \centering
        \includegraphics[width=\linewidth]{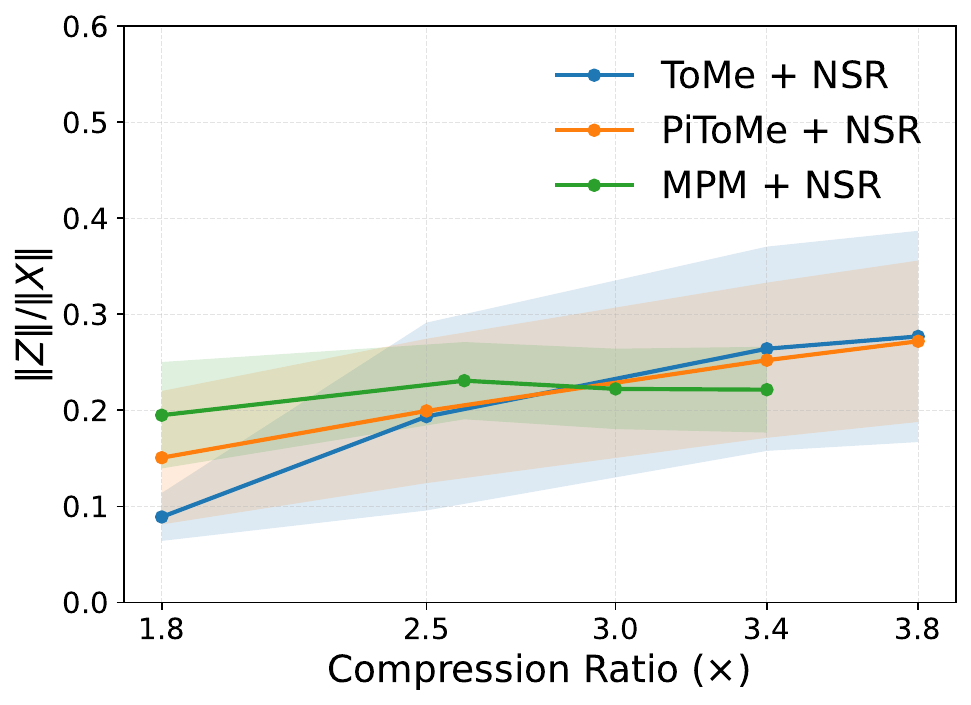}
        \vspace{-0.8em}
        \centerline{\small (c) Cityscapes}
    \end{minipage}
    \caption{Relative null-space residual magnitude
    \(\rho_Z=\|Z\|_F/\|X\|_F\) for the trained \(+\method\) token-merging
    models on (a) ADE20K, (b) Pascal VOC, and (c) Cityscapes. For each compression setting, the solid line reports the mean across merging blocks in which NSR is constructed, and the shaded region shows the corresponding mean $\pm$ one standard deviation.
}
    \label{fig:nsr-residual-magnitude}
\end{figure*}

Figure~\ref{fig:nsr-residual-magnitude} shows that the relative residual
magnitude generally increases with stronger compression for ToMe and PiToMe
across the three datasets, whereas MPM varies more mildly and is occasionally
non-monotonic. These results show that the null-space residual is numerically
non-negligible and that its magnitude depends on both the merging operator
and the compression strength.

\begin{table}[t]
\centering
\caption{Maximum null-space violation \(\epsilon_{\mathrm{null}}=\|AZ\|_\infty\)
for the trained \(+\method\) models. Each entry is the maximum observed value
over all evaluated compression settings and merging blocks.}
\label{tab:nsr-null-violation}
\small
\setlength{\tabcolsep}{5pt}
\begin{tabular}{lccc}
\toprule
Dataset & ToMe & PiToMe & MPM \\
\midrule
ADE20K
& \(9.02\times10^{-7}\)
& \(9.54\times10^{-7}\)
& \(9.54\times10^{-7}\) \\
Pascal VOC
& \(7.38\times10^{-7}\)
& \(6.28\times10^{-7}\)
& \(4.77\times10^{-7}\) \\
Cityscapes
& \(1.76\times10^{-6}\)
& \(9.54\times10^{-7}\)
& \(4.77\times10^{-7}\) \\
\bottomrule
\end{tabular}
\end{table}

Table~\ref{tab:nsr-null-violation} reports the complementary constraint check.
Across all datasets and merging methods, the maximum observed violation is on
the order of \(10^{-6}\) or below, with an overall maximum of
\(1.76\times10^{-6}\). Thus, the constructed residuals satisfy
\(AZ\approx0\) to a small numerical tolerance across the evaluated settings.
Together with Figure~\ref{fig:nsr-residual-magnitude}, these results verify
that NSR constructs residual components with non-negligible magnitude while
remaining effectively invisible to the original merging operator.

\section{Additional Performance and Parameter Results for Token Merging}
\label{app:token-additional-results}

Table~\ref{tab:token-additional} supplements the main token-merging
results with pixel accuracy (Pixel Acc.), model parameter counts
(Params), and Compression Gap Closure for ToMe, PiToMe, and MPM
across different compression settings on PASCAL VOC, Cityscapes,
and ADE20K.
These results provide an additional view of the performance recovery
achieved by NSR and its associated parameter overhead.

In terms of model size, all Original Token Merging models have the
same number of parameters as the corresponding uncompressed Full-ViT,
consistent with the design of these methods, which introduce no
additional learnable parameters.
This agreement also provides an auxiliary check on the correct
implementation of the Token Merging baselines from the perspective
of model size.
After introducing NSR, the relative parameter increases are
\(3.56\%\)--\(10.23\%\), \(1.78\%\)--\(9.61\%\), and
\(4.25\%\)--\(9.90\%\) on PASCAL VOC, Cityscapes, and ADE20K,
respectively, indicating that NSR introduces only limited additional
parameter overhead.
In particular, MPM incurs lower overhead in some settings; for example,
at the \(1.8\times\) setting on Cityscapes, its parameter count
increases by only \(1.78\%\).

In terms of segmentation performance, NSR achieves higher Pixel Acc.
than the corresponding compressed baseline in 35 of the 36 evaluated
settings, with larger observed recovery under stronger compression.
For example, for ToMe at the \(3.8\times\) setting on PASCAL VOC
and ADE20K, NSR increases the parameter count by only \(8.56\%\)
and \(9.90\%\), while improving Pixel Acc. by \(0.1376\) and
\(0.2543\), respectively.
These results further show that NSR can effectively mitigate the
segmentation-performance degradation caused by Token Merging under
strong compression while introducing limited additional parameter
overhead.

To complement the absolute mIoU gains in Table~\ref{tab:token-main},
we report Compression Gap Closure, which measures the fraction
of the performance gap between the original compressed baseline
and Full-ViT that is closed by NSR.
For the evaluated configurations, where
\(\mathrm{mIoU}_{\mathrm{Full}}
> \mathrm{mIoU}_{\mathrm{Base}}\), it is defined as
\begin{equation}
\label{eq:token-gap-closure}
\mathrm{Gap\ Closure}
=
\frac{
\mathrm{mIoU}_{\mathrm{NSR}}
-
\mathrm{mIoU}_{\mathrm{Base}}
}{
\mathrm{mIoU}_{\mathrm{Full}}
-
\mathrm{mIoU}_{\mathrm{Base}}
}
\times 100\%,
\end{equation}
where \(\mathrm{mIoU}_{\mathrm{Full}}\),
\(\mathrm{mIoU}_{\mathrm{Base}}\), and
\(\mathrm{mIoU}_{\mathrm{NSR}}\) denote the mIoU of Full-ViT,
the corresponding original compressed baseline, and its
NSR-augmented counterpart, respectively.
A Gap Closure of \(100\%\) indicates that the NSR-augmented model
matches Full-ViT; values above \(100\%\) indicate that it surpasses
Full-ViT, while negative values indicate a decrease relative to
the compressed baseline.

The final column of Table~\ref{tab:token-additional} reports
Gap Closure for all \(36\) compression configurations.
NSR closes a substantial fraction of the performance gap across
multiple merging methods and datasets.
For example, ToMe at the \(3.4\times\) operating point on PASCAL VOC
gains \(31.51\) mIoU points, closing \(67.7\%\) of the gap to Full-ViT.
On ADE20K, PiToMe at \(1.8\times\) reaches a Gap Closure of
\(127.1\%\), consistent with its mIoU exceeding Full-ViT.

Because the metric normalizes by the Full-ViT--baseline performance
gap, it is sensitive to small reference gaps.
For example, ToMe at \(1.8\times\) on Cityscapes has a Gap Closure
of \(-33.3\%\), although its absolute mIoU change is only \(-0.04\).
Gap Closure is therefore interpreted together with the absolute
mIoU values and gains in Table~\ref{tab:token-main}.

\begin{table*}[t]
\centering
\caption{
Additional semantic-segmentation results for token merging
with and without NSR.
WP denotes the computation-reduction operating point of the
original compressed baseline relative to Full-ViT.
Pixel Acc. and Params are reported separately for Original
and +NSR.
Within each Original--NSR pair, the higher Pixel Acc. is shown
in bold.
Gap Closure is computed from unrounded mIoU values according
to Eq.~\eqref{eq:token-gap-closure} and reported to one decimal place.
}
\label{tab:token-additional}
\small
\setlength{\tabcolsep}{4pt}
\renewcommand{\arraystretch}{0.96}
\begin{tabular}{llc cc cc c}
\toprule
& & & \multicolumn{2}{c}{Pixel Acc. $\uparrow$}
& \multicolumn{2}{c}{Params}
& \\
\cmidrule(lr){4-5}\cmidrule(lr){6-7}
Dataset & Method & WP
& Original & +NSR
& Original & +NSR
& Gap Closure $\uparrow$ \\
\midrule

\multirow{13}{*}{VOC}
& Full-ViT & \(1.0\times\)
& 0.9062 & --
& 5,528,853 & --
& -- \\
\cmidrule(lr){2-8}

& \multirow{4}{*}{ToMe}
& \(1.8\times\)
& 0.8988 & \textbf{0.9030}
& 5,528,853 & 6,094,485
& \(69.7\%\) \\
& & \(2.5\times\)
& 0.8248 & \textbf{0.8873}
& 5,528,853 & 6,094,485
& \(73.8\%\) \\
& & \(3.4\times\)
& 0.7131 & \textbf{0.8607}
& 5,528,853 & 6,048,373
& \(67.7\%\) \\
& & \(3.8\times\)
& 0.7111 & \textbf{0.8487}
& 5,528,853 & 6,002,261
& \(61.1\%\) \\
\cmidrule(lr){2-8}

& \multirow{4}{*}{PiToMe}
& \(1.8\times\)
& 0.8934 & \textbf{0.8978}
& 5,528,853 & 6,094,485
& \(28.5\%\) \\
& & \(2.5\times\)
& 0.8722 & \textbf{0.8869}
& 5,528,853 & 6,094,485
& \(37.6\%\) \\
& & \(3.4\times\)
& 0.8241 & \textbf{0.8696}
& 5,528,853 & 6,094,485
& \(48.8\%\) \\
& & \(3.8\times\)
& 0.7978 & \textbf{0.8595}
& 5,528,853 & 6,094,485
& \(48.5\%\) \\
\cmidrule(lr){2-8}

& \multirow{4}{*}{MPM}
& \(1.8\times\)
& 0.8969 & \textbf{0.8991}
& 5,528,853 & 5,725,589
& \(22.7\%\) \\
& & \(2.4\times\)
& 0.8866 & \textbf{0.8893}
& 5,528,853 & 5,817,813
& \(14.5\%\) \\
& & \(3.0\times\)
& 0.8735 & \textbf{0.8827}
& 5,528,853 & 6,002,261
& \(27.2\%\) \\
& & \(3.3\times\)
& 0.8601 & \textbf{0.8753}
& 5,528,853 & 6,094,485
& \(26.0\%\) \\

\midrule

\multirow{13}{*}{Cityscapes}
& Full-ViT & \(1.0\times\)
& 0.9517 & --
& 5,884,051 & --
& -- \\
\cmidrule(lr){2-8}

& \multirow{4}{*}{ToMe}
& \(1.8\times\)
& \textbf{0.9508} & 0.9507
& 5,884,051 & 6,449,683
& \(-33.3\%\) \\
& & \(2.5\times\)
& 0.9390 & \textbf{0.9497}
& 5,884,051 & 6,449,683
& \(79.3\%\) \\
& & \(3.4\times\)
& 0.9046 & \textbf{0.9463}
& 5,884,051 & 6,449,683
& \(86.4\%\) \\
& & \(3.8\times\)
& 0.8860 & \textbf{0.9446}
& 5,884,051 & 6,449,683
& \(85.0\%\) \\
\cmidrule(lr){2-8}

& \multirow{4}{*}{PiToMe}
& \(1.8\times\)
& 0.9489 & \textbf{0.9495}
& 5,884,051 & 6,449,683
& \(4.7\%\) \\
& & \(2.5\times\)
& 0.9454 & \textbf{0.9492}
& 5,884,051 & 6,449,683
& \(80.5\%\) \\
& & \(3.4\times\)
& 0.9375 & \textbf{0.9459}
& 5,884,051 & 6,449,683
& \(53.3\%\) \\
& & \(3.8\times\)
& 0.9333 & \textbf{0.9448}
& 5,884,051 & 6,449,683
& \(62.2\%\) \\
\cmidrule(lr){2-8}

& \multirow{4}{*}{MPM}
& \(1.8\times\)
& 0.9493 & \textbf{0.9501}
& 5,884,051 & 5,988,563
& \(49.5\%\) \\
& & \(2.6\times\)
& 0.9443 & \textbf{0.9476}
& 5,884,051 & 6,219,123
& \(47.6\%\) \\
& & \(3.0\times\)
& 0.9429 & \textbf{0.9469}
& 5,884,051 & 6,265,235
& \(25.8\%\) \\
& & \(3.4\times\)
& 0.9379 & \textbf{0.9453}
& 5,884,051 & 6,449,683
& \(42.5\%\) \\

\midrule

\multirow{13}{*}{ADE20K}
& Full-ViT & \(1.0\times\)
& 0.7708 & --
& 5,712,726 & --
& -- \\
\cmidrule(lr){2-8}

& \multirow{4}{*}{ToMe}
& \(1.8\times\)
& 0.7687 & \textbf{0.7693}
& 5,712,726 & 6,278,358
& \(81.4\%\) \\
& & \(2.5\times\)
& 0.7272 & \textbf{0.7661}
& 5,712,726 & 6,278,358
& \(91.5\%\) \\
& & \(3.4\times\)
& 0.5661 & \textbf{0.7548}
& 5,712,726 & 6,278,358
& \(82.5\%\) \\

% TODO: Verify ADE20K--ToMe at 3.8x against the unrounded source results.
% The original 84.5% is retained here without an unverified correction.
% It is inconsistent with the displayed Full/Base/NSR mIoU values
% of 37.60/12.44/33.74, even after accounting for rounding.
& & \(3.8\times\)
& 0.4963 & \textbf{0.7506}
& 5,712,726 & 6,278,358
& \(84.7\%\) \\
\cmidrule(lr){2-8}

& \multirow{4}{*}{PiToMe}
& \(1.8\times\)
& 0.7662 & \textbf{0.7704}
& 5,712,726 & 6,278,358
& \(127.1\%\) \\
& & \(2.5\times\)
& 0.7545 & \textbf{0.7636}
& 5,712,726 & 6,278,358
& \(41.4\%\) \\
& & \(3.4\times\)
& 0.7332 & \textbf{0.7601}
& 5,712,726 & 6,278,358
& \(68.6\%\) \\
& & \(3.8\times\)
& 0.7160 & \textbf{0.7532}
& 5,712,726 & 6,278,358
& \(54.5\%\) \\
\cmidrule(lr){2-8}

& \multirow{4}{*}{MPM}
& \(1.8\times\)
& 0.7685 & \textbf{0.7713}
& 5,712,726 & 5,955,574
& \(-7.0\%\) \\
& & \(2.4\times\)
& 0.7653 & \textbf{0.7659}
& 5,712,726 & 6,001,686
& \(33.6\%\) \\
& & \(3.2\times\)
& 0.7486 & \textbf{0.7622}
& 5,712,726 & 6,186,134
& \(53.7\%\) \\
& & \(3.4\times\)
& 0.7432 & \textbf{0.7583}
& 5,712,726 & 6,278,358
& \(54.8\%\) \\

\bottomrule
\end{tabular}
\end{table*}

\section{Label-Based Null-Space Diagnostic:
Definition and Supplementary Analysis}
\label{app:trnsil}

This appendix provides the full definition of the diagnostic
score used in the main text, the binning procedure, and
supplementary results on accuracy gains.

For a local aggregation event \(u\), let
\(S_u \in \mathbb{R}^{k_u \times C}\) contain the one-hot
label encodings of the participating members, where \(k_u\)
is the number of members and \(C\) is the number of classes.
Let \(t_u \in \mathbb{R}^{k_u}\) be the nonzero aggregation
coefficient vector actually used by the corresponding backbone.
Following the graph-aggregation residual projection in
Eq.~\eqref{eq:graph_nsr}, we define
\begin{equation}
    P_u
    =
    I_{k_u}
    -
    \frac{t_u t_u^\top}{t_u^\top t_u},
    \qquad
    r_u
    =
    \frac{\lVert P_u S_u \rVert_F^2}
         {\lVert S_u \rVert_F^2 + \epsilon},
    \label{eq:local-trnsil}
\end{equation}
where \(\epsilon > 0\) is a numerical stabilizer.
Since \(P_u\) is the orthogonal projection onto
\(\ker(t_u^\top)\), it satisfies
\(t_u^\top P_u S_u = 0\).
Thus, \(r_u\) measures the normalized projection energy of
the label encodings in the null space of this local aggregation.
Both the member set and the aggregation coefficients follow
the aggregation definition of the corresponding backbone.
This score operates on label encodings and does not directly
measure task-relevant information loss in the model's hidden
features.

For a model with \(L\) message-passing layers, let
\(\mathcal{T}_v^{(L)}\) denote the set of local aggregation
events in the unfolded computation tree of target node \(v\).
We define the node-level diagnostic score as
\begin{equation}
    R_v^{(L)}
    =
    \frac{
        \sum_{u \in \mathcal{T}_v^{(L)}} k_u r_u
    }{
        \sum_{u \in \mathcal{T}_v^{(L)}} k_u
    },
    \label{eq:trnsil}
\end{equation}
where \(u\) denotes a local aggregation event in the computation
tree and \(k_u\) is the number of members in that event.
This weighted average summarizes the local diagnostic scores
within the target node's computation tree.
All scores are computed from ground-truth labels after training
and are not used for training, hyperparameter selection,
or model selection.

On Roman-empire, we compute \(R_v^{(L)}\) for the test nodes
separately for each official split and each of the GCN, GIN,
and GraphSAGE backbones.
Bin boundaries are determined by score deciles, with nodes
having identical scores kept in the same bin.
The Original and \(+\)NSR models are compared on the same
nodes within each bin.
Bins are constructed independently within each split, and
cross-split statistics are aggregated by bin index, ordered
from low to high scores.

Figure~\ref{fig:trnsil-accuracy} in the main text reports
the mean accuracy in each bin.
Figure~\ref{fig:trnsil-gain} further shows the accuracy gain
within the same bins.
For bin \(b\) in split \(s\), we define
\begin{equation}
    \Delta\mathrm{Acc}_{s,b}
    =
    \mathrm{Acc}_{+\mathrm{NSR},s,b}
    -
    \mathrm{Acc}_{\mathrm{Original},s,b},
    \label{eq:trnsil-gain}
\end{equation}
and average \(\Delta\mathrm{Acc}_{s,b}\) over the ten official
splits.
For all three backbones, the mean gain exhibits an overall
increasing trend with the score-bin index, with local
fluctuations.
This figure presents the differences underlying the accuracy
results in the main text and provides a complementary view
of how NSR gains are distributed across score bins in this
experiment.

\begin{figure*}[t]
    \centering
    \includegraphics[width=0.7\linewidth]{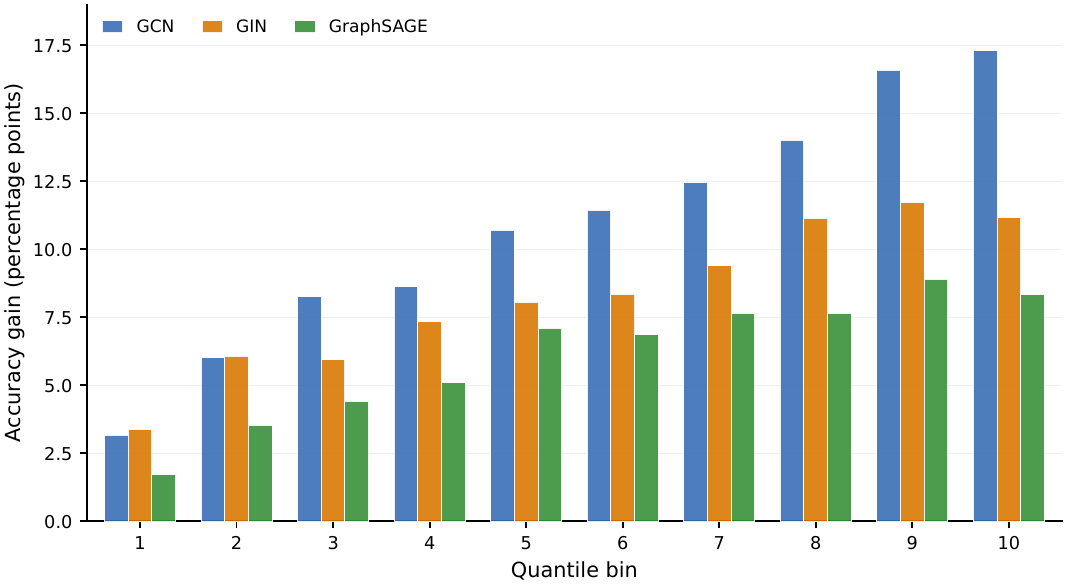}
    \caption{
        Mean accuracy gain from NSR across diagnostic-score bins
        on Roman-empire.
        Bins are the same as those in
        Figure~\ref{fig:trnsil-accuracy}.
        Within each official split, the gain is computed as
        the accuracy of the NSR-augmented model minus that
        of the Original model.
        The figure reports the mean over the ten official splits.
    }
    \label{fig:trnsil-gain}
\end{figure*}

% ============================================================

% ============================================================

\section{Operating-Point Construction}
\label{app:token-operating-points}

We define the computation-reduction factor as
\begin{equation}
    C =
    \frac{F_{\mathrm{Full}}}
         {F_{\mathrm{compressed}}},
    \label{eq:token-compression-factor}
\end{equation}
where $F_{\mathrm{Full}}$ and $F_{\mathrm{compressed}}$ denote
the GFLOPs of the uncompressed Full-ViT and the corresponding
token-merging baseline, respectively. For the token-merging experiments, the reported WP values
serve as nominal labels for configurations constructed using
this ratio; the achieved computation-reduction factors,
which need not equal these labels, are reported in the
Achieved columns of the configuration tables below.
For a target reduction factor $C_{\mathrm{target}}$, each method
searches only within its native compression-control space to
match the target computation budget as closely as possible,
without modifying its original matching or merging mechanism.

\paragraph{ToMe.}
ToMe controls compression through the planned number of tokens
$r$ merged per layer.
For each target computation budget, we search over all valid
integer values of $r$ and select
\begin{equation}
    r^{*} =
    \underset{r}{\arg\min}
    \left|
        \frac{F_{\mathrm{Full}}}
             {F_{\mathrm{ToMe}}(r)}
        - C_{\mathrm{target}}
    \right|,
    \label{eq:tome-operating-point}
\end{equation}
where $F_{\mathrm{ToMe}}(r)$ denotes the GFLOPs of the ToMe
baseline with compression parameter $r$.
The actual number of merged tokens $r_l$ at layer $l$ follows
the original ToMe merging rules and is constrained by the
number of remaining tokens.

\paragraph{PiToMe.}
PiToMe controls compression through the retention ratio $q$.
The number of tokens merged at Transformer block $l$ is
\begin{equation}
    r_l =
    \left\lfloor T_l(1-q) \right\rfloor,
    \label{eq:pitome-merge-count}
\end{equation}
where $T_l$ denotes the number of patch tokens before merging
at that block, excluding the CLS token.
We deterministically search $q$ over $[0.5,0.999]$ with a
step size of $0.001$, while keeping the original energy-based
token selection, margin schedule, and matching mode schedule
unchanged.

\paragraph{MPM.}
Since the actual number of merged tokens in MPM depends on
the input features, we construct different operating points
by searching the insertion-block schedule without modifying
its mutual-nearest-neighbor matching or arithmetic-mean
merging mechanism.
We use 32 and 200 images from the training split for coarse
search and refinement, respectively.
Candidate schedules are generated by adding, removing, or
moving insertion blocks, and the final schedule is selected
according to its distance from the target computation-reduction
factor, without using validation mIoU.

Because MPM matching depends on the input content and model
representations, the compression factors estimated during
calibration and achieved during final validation may differ.
We report both values and keep the selected insertion schedule
fixed during final evaluation.

\paragraph{Configuration consistency.}
For each operator and operating point, the baseline and its
NSR-augmented counterpart share the same compression-control
configuration.
ToMe shares $r$ and the layer-wise merge schedule;
PiToMe shares the retention ratio, layer-wise merge schedule,
margin schedule, and matching mode schedule;
and MPM shares the insertion-block schedule.
NSR does not search for a more aggressive compression
configuration to compensate for its additional computational
overhead.

All Transformer block indices reported below are zero-based.

% ============================================================
% ToMe Compression Configurations
% ============================================================

\subsection{ToMe Compression Configurations}
\label{app:tome-configurations}

Table~\ref{tab:tome-compression-configurations} reports
the selected $r$, achieved computation-reduction factor,
and GFLOPs for ToMe on the three datasets.
Since the actual layer-wise merge count is constrained by
the number of remaining tokens, the selected $r$ may not
be fully applied at every layer.
The corresponding layer-wise compression distributions are
visualized in Fig.~\ref{fig:layerwise-token-merging}.

\begin{table}[t]
\centering
\caption{ToMe compression configurations on the three semantic
segmentation datasets. Target and Achieved denote the target
and realized computation-reduction factors, respectively.
Selected $r$ is the planned number of tokens merged per layer.
GFLOPs denotes the computation of the corresponding baseline.}
\label{tab:tome-compression-configurations}
\small
\setlength{\tabcolsep}{7pt}
\renewcommand{\arraystretch}{0.9}
\begin{tabular}{lcccc}
\toprule
Dataset & Target & Selected $r$ & Achieved & GFLOPs \\
\midrule
\multirow{4}{*}{VOC2012}
& $1.8\times$ & 14  & $1.797\times$ & 0.6980 \\
& $2.5\times$ & 21  & $2.550\times$ & 0.4918 \\
& $3.4\times$ & 29  & $3.367\times$ & 0.3725 \\
& $3.8\times$ & 33  & $3.768\times$ & 0.3329 \\
\midrule
\multirow{4}{*}{ADE20K}
& $1.8\times$ & 66  & $1.799\times$ & 5.8169 \\
& $2.5\times$ & 98  & $2.490\times$ & 4.2015 \\
& $3.4\times$ & 140 & $3.392\times$ & 3.0850 \\
& $3.8\times$ & 161 & $3.805\times$ & 2.7500 \\
\midrule
\multirow{4}{*}{Cityscapes}
& $1.8\times$ & 124 & $1.798\times$ & 16.9827 \\
& $2.5\times$ & 187 & $2.505\times$ & 12.1911 \\
& $3.4\times$ & 266 & $3.399\times$ & 8.9841 \\
& $3.8\times$ & 304 & $3.799\times$ & 8.0367 \\
\bottomrule
\end{tabular}
\end{table}

% ============================================================
% PiToMe Compression Configurations
% ============================================================

\subsection{PiToMe Compression Configurations}
\label{app:pitome-configurations}

PiToMe uses the same margin schedule across all datasets
and operating points:
\begin{equation}
\boldsymbol{m} = [0.75,\ 0.6875,\ 0.625,\ 0.5625,\ 0.5,\ 0.4375,\ 0.375,\ 0.3125,\ 0.25,\ 0.1875,\ 0.125,\ 0.0625].
\label{eq:pitome-margin-schedule}
\end{equation}
Here, $\boldsymbol{m}$ specifies the margin parameter of each
Transformer block in execution order.
The first six blocks use BSM mode, while the remaining six
blocks use standard mode.
The CLS token never participates in matching or merging.

Table~\ref{tab:pitome-compression-configurations} summarizes
the retention ratio, achieved computation-reduction factor,
and GFLOPs for each operating point.
The corresponding layer-wise compression distributions are
shown in Fig.~\ref{fig:layerwise-token-merging}.

\begin{table}[t]
\centering
\caption{PiToMe compression configurations on the three semantic
segmentation datasets. $q$ denotes the retention ratio.
All operating points use the same margin and matching mode
schedules. GFLOPs denotes the computation of the corresponding
baseline.}
\label{tab:pitome-compression-configurations}
\small
\setlength{\tabcolsep}{7pt}
\renewcommand{\arraystretch}{0.96}
\begin{tabular}{lcccc}
\toprule
Dataset & Target & Retention ratio $q$ & Achieved & GFLOPs \\
\midrule
\multirow{4}{*}{VOC2012}
& $1.8\times$ & 0.895 & $1.796\times$ & 0.6983 \\
& $2.5\times$ & 0.829 & $2.503\times$ & 0.5010 \\
& $3.4\times$ & 0.754 & $3.407\times$ & 0.3681 \\
& $3.8\times$ & 0.725 & $3.783\times$ & 0.3316 \\
\midrule
\multirow{4}{*}{ADE20K}
& $1.8\times$ & 0.908 & $1.804\times$ & 5.8011 \\
& $2.5\times$ & 0.849 & $2.503\times$ & 4.1799 \\
& $3.4\times$ & 0.781 & $3.401\times$ & 3.0765 \\
& $3.8\times$ & 0.752 & $3.795\times$ & 2.7572 \\
\midrule
\multirow{4}{*}{Cityscapes}
& $1.8\times$ & 0.913 & $1.806\times$ & 16.9101 \\
& $2.5\times$ & 0.859 & $2.500\times$ & 12.2127 \\
& $3.4\times$ & 0.794 & $3.407\times$ & 8.9615 \\
& $3.8\times$ & 0.767 & $3.797\times$ & 8.0407 \\
\bottomrule
\end{tabular}
\end{table}

% ============================================================
% MPM Compression Configurations
% ============================================================

\subsection{MPM Compression Configurations}
\label{app:mpm-configurations}

MPM constructs different computation operating points by
adjusting its insertion-block schedule.
Table~\ref{tab:mpm-compression-configurations} summarizes
the target reduction factors, selected insertion blocks,
and compression factors achieved during calibration and
final validation.

\begin{table}[t]
\centering
\caption{MPM compression configurations on the three semantic
segmentation datasets. Target denotes the search target,
while Reported point is the operating-point label used in
the final experiments. Search and Achieved denote the
computation-reduction factors obtained during calibration
and final validation, respectively. GFLOPs denotes the
computation of the final baseline. A range $i$--$j$ in
Insertion blocks includes all consecutive block indices
from $i$ to $j$, inclusive.}
\label{tab:mpm-compression-configurations}
\small
\setlength{\tabcolsep}{3pt}
\renewcommand{\arraystretch}{0.7}
\resizebox{0.7\linewidth}{!}{%
\begin{tabular}{lcclccc}
\toprule
Dataset
& Target
& \shortstack{Reported\\point}
& Insertion blocks
& Search
& Achieved
& GFLOPs \\
\midrule
\multirow{4}{*}{VOC2012}
& $1.8\times$
& $1.8\times$
& [0,4,5,9]
& $1.801\times$
& $1.793\times$
& 0.6994 \\
& $2.5\times$
& $2.4\times$
& [0--1,3,5--6,10]
& $2.494\times$
& $2.449\times$
& 0.5122 \\
& $3.0\times$
& $3.0\times$
& [0--2,4--6,8--11]
& $3.004\times$
& $2.926\times$
& 0.4286 \\
& $3.8\times$
& $3.3\times$
& [0--11]
& $3.438\times$
& $3.332\times$
& 0.3764 \\
\midrule
\multirow{4}{*}{ADE20K}
& $1.8\times$
& $1.8\times$
& [2,4,7,9,11]
& $1.803\times$
& $1.803\times$
& 5.8041 \\
& $2.5\times$
& $2.4\times$
& [0--1,4,7,9,11]
& $2.507\times$
& $2.355\times$
& 4.4431 \\
& $3.4\times$
& $3.2\times$
& [0--4,7--11]
& $3.412\times$
& $3.167\times$
& 3.3035 \\
& $3.8\times$
& $3.4\times$
& [0--11]
& $3.691\times$
& $3.425\times$
& 3.0553 \\
\midrule
\multirow{4}{*}{Cityscapes}
& $1.8\times$
& $1.8\times$
& [0,2]
& $1.800\times$
& $1.816\times$
& 16.8170 \\
& $2.5\times$
& $2.6\times$
& [1--6,10]
& $2.487\times$
& $2.562\times$
& 11.9181 \\
& $3.0\times$
& $3.0\times$
& [0--2,4--7,10]
& $2.966\times$
& $3.033\times$
& 10.0667 \\
& $3.4\times$
& $3.4\times$
& [0--11]
& $3.271\times$
& $3.359\times$
& 9.0903 \\
\bottomrule
\end{tabular}%
}
\end{table}

As shown in Table~\ref{tab:mpm-compression-configurations},
the final computation-reduction factor of MPM may differ
from its calibration estimate.
For example, the $2.5\times$ target on ADE20K achieves
$2.507\times$ during calibration but $2.355\times$ during
final validation.
This difference is associated with the dependence of
content-adaptive matching on the input and model
representations, while the selected insertion schedule
remains unchanged during final evaluation.

For the $3.8\times$ targets on VOC2012 and ADE20K, and the
$3.4\times$ target on Cityscapes, all 12 Transformer blocks
are used as MPM insertion points.
These configurations exhaust the available insertion
positions under the current experimental protocol.
We therefore report their achieved compression factors
rather than treating the original targets as exactly matched.

% ============================================================
% Layer-wise Token Merging Distributions
% ============================================================

\subsection{Layer-wise Token Merging Distributions}
\label{app:layerwise-token-merging}

To further illustrate the layer-wise compression behavior
across operators and operating points,
Fig.~\ref{fig:layerwise-token-merging} presents the normalized
token-merging distributions of ToMe and PiToMe on the three
datasets.
The normalized number of merged tokens at Transformer
block $l$ is defined as
\begin{equation}
    M_l = \frac{r_l}{T_0},
    \label{eq:normalized-merge-count}
\end{equation}
where $r_l$ denotes the actual number of tokens merged
at block $l$, and $T_0$ denotes the initial number of
patch tokens entering the ViT, excluding the CLS token.
This normalization enables comparison of layer-wise
merging quantities across different input resolutions
on a common scale.

\begin{figure*}[t]
    \centering
    \includegraphics[width=0.98\textwidth]
    {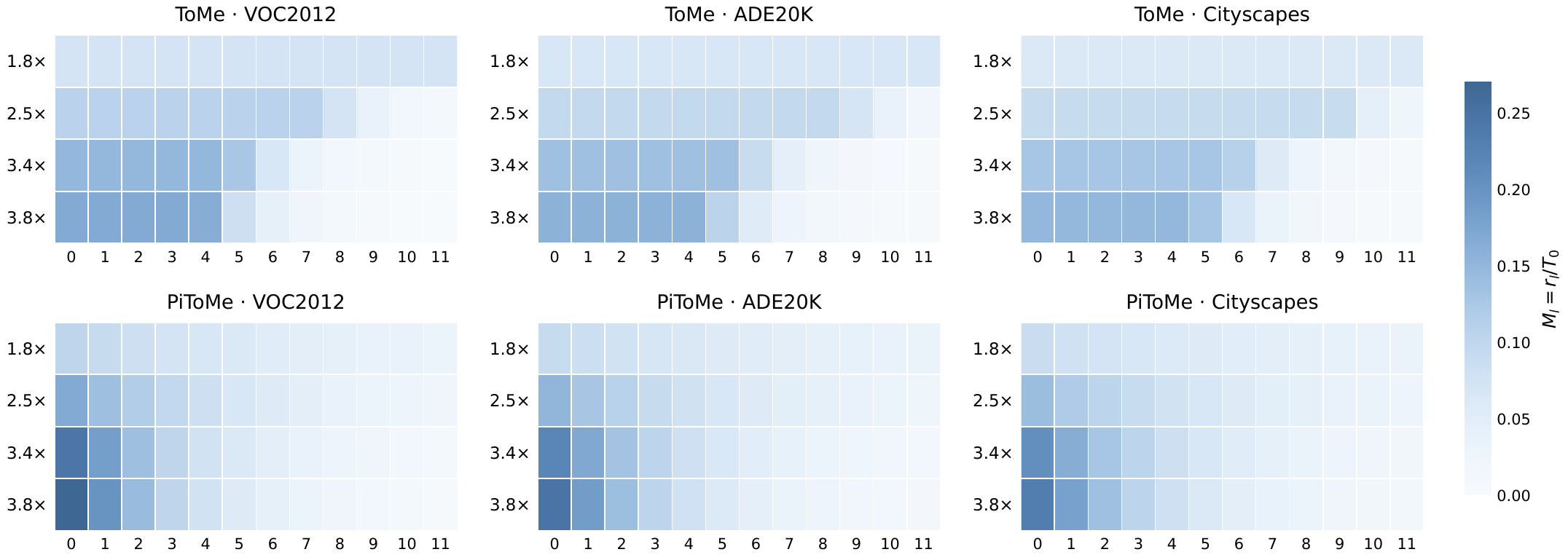}
    \caption{Layer-wise token-merging distributions of ToMe
    and PiToMe on VOC2012, ADE20K, and Cityscapes.
    The top and bottom rows correspond to ToMe and PiToMe,
    respectively, while the three columns correspond to
    VOC2012, ADE20K, and Cityscapes from left to right.
    The horizontal axis indicates the Transformer block
    index, and the vertical axis indicates the target
    computation-reduction factor.
    Color represents the normalized number of merged tokens,
    $M_l=r_l/T_0$.
    All subplots share the same color scale, with darker
    colors indicating a larger fraction of the initial
    patch tokens merged at the corresponding block.}
    \label{fig:layerwise-token-merging}
\end{figure*}

As shown in Fig.~\ref{fig:layerwise-token-merging},
stronger target compression leads to greater token reduction
in the early layers for both ToMe and PiToMe, while their
layer-wise merging patterns differ.
ToMe maintains a relatively stable merge count in the
initial blocks before the remaining token count imposes
a tighter upper bound.
In contrast, PiToMe uses a shared retention ratio, resulting
in progressively decreasing merge counts as the remaining
token sequence becomes shorter.
Similar patterns are observed across the three datasets,
indicating that the NSR evaluation covers token-merging
configurations with different layer-wise compression
patterns rather than a single reduction pattern.

% ============================================================
\section{Additional Comparisons and Ablations}
\label{app:additional}

\subsection{Comparison with Post-Merge-Only Refinement}
\label{app:residual-ablation}

NSR is motivated by a potential mismatch between operator-induced
indistinguishability and task-required distinctions:
member-level variations invisible to an operator may still be
needed for downstream prediction.
This comparison examines whether additional processing of
operator-visible merged representations can reproduce the gains
of the complete NSR pathway.

To this end, we introduce a post-merge-only control based on
merged representations.
It takes the output \(y_G\) of each merge group as input,
generates a feature update through encoding, gating, and decoding,
and adds this update back to the corresponding merged representation.
This branch only processes the group representations already
produced by the operator.
It does not access pre-merge member representations or null-space
residuals, nor does it introduce spatial memory or a write-back path.
All models use the same backbone architecture, training protocol,
merging rules, and predefined compression configurations.

In NSR, the null space of the currently realized linear operator
defines the source of candidate complementary residuals, while
the encoder, member-wise gating, and application-specific
integration jointly provide a task-supervised mechanism for
exploiting these residuals.
Accordingly, this comparison evaluates the complete pathway
against local refinement of the merged representations.

As shown in Table~\ref{tab:residual-ablation}, the effect of
the post-merge-only refinement varies across compressors and operating
points: it provides small improvements in some settings but
reduces baseline performance in others.
In contrast, Full NSR outperforms both the corresponding baseline
and the post-merge-only control in every setting reported in
the table, with more pronounced gains in several aggressive
compression settings.
For example, at the \(3.4\times\) operating point of ToMe,
the post-merge-only refinement decreases mIoU from the baseline value
of \(0.1831\) to \(0.1778\), whereas Full NSR increases it
to \(0.4982\).

These results show that, in the evaluated settings, additional
local processing of the merged representations alone does
not reproduce the gains of the complete NSR pathway.
Spatial routing is specific to the token-merging instantiation
for dense prediction and is absent from the graph-aggregation
instantiation, where NSR also improves performance across
multiple backbones and tasks, providing complementary evidence
for the generality of the same design principle across
structurally different settings.
Section~\ref{app:feature-spatial-ablation} examines how the
residuals are integrated through feature updates and spatial
routing, separately and jointly.
Section~\ref{app:source-ablation} further compares explicitly
extracted residuals with raw member features while retaining
both integration paths.

\begin{table}[t]
    \centering
    \caption{Comparison with a local residual correction based only on merged representations. Best and second-best
    results at each operating point are shown in bold and underlined,
    respectively.}
    \label{tab:residual-ablation}
    \small
    \setlength{\tabcolsep}{5pt}
    \begin{tabular}{lcccc}
        \toprule
        Compressor & Compression & Base & Post-merge-only & Full NSR \\
        \midrule

        \multirow{4}{*}{ToMe}
        & $1.8\times$ & \underline{0.6281} & 0.6274 & \textbf{0.6423} \\
        & $2.5\times$ & \underline{0.4458} & 0.4395 & \textbf{0.5954} \\
        & $3.4\times$ & \underline{0.1831} & 0.1778 & \textbf{0.4982} \\
        & $3.8\times$ & \underline{0.1698} & 0.1683 & \textbf{0.4623} \\

        \midrule

        \multirow{4}{*}{PiToMe}
        & $1.8\times$ & 0.6148 & \underline{0.6181} & \textbf{0.6244} \\
        & $2.5\times$ & 0.5577 & \underline{0.5593} & \textbf{0.5919} \\
        & $3.4\times$ & \underline{0.4370} & 0.4198 & \textbf{0.5402} \\
        & $3.8\times$ & \underline{0.3732} & 0.3683 & \textbf{0.5066} \\

        \midrule

        \multirow{4}{*}{MPM}
        & $1.8\times$ & 0.6199 & \underline{0.6203} & \textbf{0.6264} \\
        & $2.4\times$ & \underline{0.5887} & 0.5837 & \textbf{0.5974} \\
        & $3.0\times$ & 0.5493 & \underline{0.5609} & \textbf{0.5763} \\
        & $3.3\times$ & 0.5134 & \underline{0.5225} & \textbf{0.5485} \\

        \bottomrule
    \end{tabular}
\end{table}

\subsection{Ablation of information-integration pathways within NSR}
\label{app:feature-spatial-ablation}

We examine Feature Integration and Spatial Routing as two
application-specific ways of using null-space residual information in
token merging.
Copy-based unmerging restores the original token layout by
copying each group representation to its member positions,
but does not recover the original feature differences
between those members.
Both NSR variants extract member null-space residuals and
process them through the residual encoder and member-dependent
gates.
Feature Integration uses the resulting codes to update merged
tokens, allowing member differences to influence subsequent
backbone computation.
Spatial Routing uses the member codes to generate feature
corrections at the corresponding original patch positions.
We evaluate the two paths separately in NSR-Feature and NSR-Spatial, respectively, and jointly in Full NSR. All three NSR variants use the same null-space residual construction. Comparisons among these variants therefore examine how the residuals are integrated, with the original baseline providing a reference without the NSR pathway.

As shown in Table~\ref{tab:feature-spatial-ablation}, The strong performance of NSR-Spatial in several configurations shows that spatial utilization of null-space residuals can provide substantial benefits. Full NSR achieves
the highest mIoU in ten of the twelve settings, indicating an
advantage from jointly using the two paths in most evaluated
configurations. For example, at the \(3.4\times\) and \(3.8\times\)
operating points of PiToMe, jointly enabling the two paths improves
mIoU from \(0.5008\) and \(0.4810\) with the spatial path alone to
\(0.5402\) and \(0.5066\), respectively.

At the \(2.4\times\) operating point of MPM, the feature-only and
spatial-only variants achieve mIoU values of \(0.5837\) and \(0.5778\),
respectively, neither exceeding the baseline value of \(0.5887\),
whereas Full NSR reaches \(0.5974\). These results support a
complementary role for the two paths in some configurations:
jointly using null-space residual information through local feature
correction and compensation at the corresponding spatial positions
can yield gains not achieved by either path alone.

\begin{table}[t]
    \centering
    \caption{Ablation of Feature Integration and Spatial Routing. Best and
    second-best results at each operating point are shown in bold and
    underlined, respectively.}
    \label{tab:feature-spatial-ablation}
    \small
    \setlength{\tabcolsep}{3.5pt}
    \begin{tabular}{lccccc}
        \toprule
        Compressor & Compression & Base & NSR-Feature & NSR-Spatial & Full NSR \\
        \midrule

        \multirow{4}{*}{ToMe}
        & $1.8\times$ & 0.6281 & \underline{0.6351} & 0.6308 & \textbf{0.6423} \\
        & $2.5\times$ & 0.4458 & 0.4347 & \textbf{0.5985} & \underline{0.5954} \\
        & $3.4\times$ & 0.1831 & 0.1693 & \underline{0.4823} & \textbf{0.4982} \\
        & $3.8\times$ & 0.1698 & 0.1632 & \underline{0.4576} & \textbf{0.4623} \\

        \midrule

        \multirow{4}{*}{PiToMe}
        & $1.8\times$ & 0.6148 & 0.6138 & \textbf{0.6258} & \underline{0.6244} \\
        & $2.5\times$ & 0.5577 & 0.5687 & \underline{0.5873} & \textbf{0.5919} \\
        & $3.4\times$ & 0.4370 & 0.4162 & \underline{0.5008} & \textbf{0.5402} \\
        & $3.8\times$ & 0.3732 & 0.3783 & \underline{0.4810} & \textbf{0.5066} \\

        \midrule

        \multirow{4}{*}{MPM}
        & $1.8\times$ & 0.6199 & 0.6187 & \underline{0.6211} & \textbf{0.6264} \\
        & $2.4\times$ & \underline{0.5887} & 0.5837 & 0.5778 & \textbf{0.5974} \\
        & $3.0\times$ & 0.5493 & 0.5468 & \underline{0.5647} & \textbf{0.5763} \\
        & $3.3\times$ & 0.5134 & 0.5325 & \underline{0.5367} & \textbf{0.5485} \\

        \bottomrule
    \end{tabular}
\end{table}

\subsection{Ablation of the Input Source within NSR}
\label{app:source-ablation}

We examine the input-source choice within the token-merging
instantiation of the complete NSR framework.
To this end, we construct a raw-member-source control, denoted
RawSource, by replacing $z_i = x_i - y_G$ with $x_i$ as the
source supplied to both the encoder and the source input of
the gating network.
The member and group representations that additionally
condition the gate remain unchanged.
RawSource retains the same normalization modules, encoder,
gating network, feature decoder, spatial routing, and final
spatial decoder as Full NSR.
The two variants also use the same valid-member masks and
singleton handling, parameterization, initialization,
training protocol, and predefined compression configurations.
RawSource retains within-group variations.
Full NSR versus RawSource evaluates this internal source choice,
whereas Full NSR versus Base evaluates the complete NSR method.

\begin{table}[t]
    \centering
    \caption{
        Input-source ablation within the token-merging
instantiation of NSR on Pascal VOC.
        RawSource directly encodes pre-merge member features
        while retaining the complete NSR architecture.
        Base denotes the corresponding original merging baseline.
        WP follows the operating-point labels used in the main
        experiments.
        Values are mIoU on a 0--1 scale; higher is better.
        The best result at each operating point is shown in bold.
    }
    \label{tab:source-ablation}
    \small
    \setlength{\tabcolsep}{7pt}
    \renewcommand{\arraystretch}{0.9}
    \begin{tabular}{@{}llrrr@{}}
        \toprule
        Compressor & WP & Base & RawSource & Full NSR \\
        \midrule
        ToMe
        & $1.8\times$ & 0.6281 & 0.6349 & \textbf{0.6423} \\
        & $2.5\times$ & 0.4458 & 0.5930 & \textbf{0.5954} \\
        & $3.4\times$ & 0.1831 & 0.4912 & \textbf{0.4982} \\
        & $3.8\times$ & 0.1698 & 0.4490 & \textbf{0.4623} \\
        \midrule
        PiToMe
        & $1.8\times$ & 0.6148 & 0.6201 & \textbf{0.6244} \\
        & $2.5\times$ & 0.5577 & 0.5899 & \textbf{0.5919} \\
        & $3.4\times$ & 0.4370 & 0.5034 & \textbf{0.5402} \\
        & $3.8\times$ & 0.3732 & 0.4972 & \textbf{0.5066} \\
        \midrule
        MPM
        & $1.8\times$ & 0.6199 & 0.6226 & \textbf{0.6264} \\
        & $2.4\times$ & 0.5887 & 0.5839 & \textbf{0.5974} \\
        & $3.0\times$ & 0.5493 & 0.5658 & \textbf{0.5763} \\
        & $3.3\times$ & 0.5134 & 0.5435 & \textbf{0.5485} \\
        \bottomrule
    \end{tabular}
\end{table}

As shown in Table~\ref{tab:source-ablation}, Full NSR achieves
higher mIoU than RawSource in all twelve evaluated
configurations, with improvements ranging from
$0.20$ to $3.68$ mIoU points.
For example, at the $3.4\times$ operating point of PiToMe,
Full NSR achieves $0.5402$ mIoU compared with $0.5034$
for RawSource.
RawSource also improves over the corresponding original
baseline in eleven configurations, showing that directly
retaining and processing member information is beneficial
in most evaluated settings.
The additional gains of Full NSR support the practical value
of explicit null-space residual extraction within the same
complete pathway architecture.

A possible explanation is that explicit residual extraction
reduces the need for the complementary encoder to represent
components already available through the current merged
representation.
This may allow the finite-dimensional codes to focus more
effectively on operator-invisible member variations, while
the member and group representations continue to provide
context for gating.

\subsection{Complementary Pathway and Input-Source Ablation for Graph Aggregation}
\label{app:graph-source-ablation}

We compare Base, Post-aggregation-only, RawSource, and Full NSR on
Roman-empire using GCN, GraphSAGE, and GIN.
Post-aggregation-only generates corrections solely from the current
aggregate, without accessing member representations or residuals,
while matching the trainable parameter count of Full NSR.
RawSource replaces \(Z_v\) in Eq.~\eqref{eq:graph_nsr} with the
unprojected messages \(M_v\), including the corresponding gate input,
while retaining the member set, contextual inputs, remaining branch
structure, and initialization.
All configurations share the same data splits, backbone settings,
optimization protocol, training budget, and validation-based model
selection.
Table~\ref{tab:graph-source-ablation} reports the results;
Base and Full NSR are the same as in
Table~\ref{tab:heterophily-main}.

\begin{table}[htbp]
    \centering
    \caption{
        Graph-aggregation ablations on Roman-empire.
        Values are test accuracy (\%), reported as the mean and sample
        standard deviation over the ten official splits.
        The best result in each column is shown in bold.
    }
    \label{tab:graph-source-ablation}
    \small
    \setlength{\tabcolsep}{6pt}
    \begin{tabular}{@{}lccc@{}}
        \toprule
        Configuration & GCN & GraphSAGE & GIN \\
        \midrule
        Base
        & \(72.05 \pm 0.64\)
        & \(81.59 \pm 0.55\)
        & \(74.00 \pm 0.79\) \\
        Post-aggregation-only
        & \(75.30 \pm 0.69\)
        & \(84.00 \pm 0.46\)
        & \(76.01 \pm 0.62\) \\
        RawSource
        & \(79.21 \pm 0.41\)
        & \(85.57 \pm 0.64\)
        & \(80.69 \pm 0.51\) \\
        Full NSR
        & \(\boldsymbol{82.90 \pm 0.71}\)
        & \(\boldsymbol{87.71 \pm 0.61}\)
        & \(\boldsymbol{82.25 \pm 0.65}\) \\
        \bottomrule
    \end{tabular}
\end{table}

All three backbones exhibit the same mean-performance ordering:
Base \(<\) Post-aggregation-only \(<\) RawSource \(<\) Full NSR.
Additional post-aggregation processing helps, but Full NSR exceeds
this control by \(7.60\), \(3.71\), and \(6.24\) percentage points
for GCN, GraphSAGE, and GIN, respectively.
RawSource's advantage over Post-aggregation-only supports the value
of retaining and processing member information.
Full NSR further improves over RawSource by \(3.69\), \(2.14\),
and \(1.56\) percentage points.
Since RawSource also retains member differences, this comparison
supports the practical value of explicitly extracting
operator-defined residuals within the same branch architecture.

Unlike the visual instantiation, the graph branch contains no
spatial memory or write-back path
(Figure~\ref{fig:overview}).
These results therefore complement the token-merging ablations:
both the complete pathway and its residual-source choice remain
beneficial without visual spatial integration.
Together, the findings support the NSR framework of combining
operator-defined complementary information with member-level
processing and application-specific integration.

\section{Heterophilic Dataset Statistics}
\label{app:additional-heterophily}

To complement the heterophilic graph experiments in the main text,
Table~\ref{tab:heterophily-statistics} presents the statistics of the
five main heterophilic graph benchmarks, including graph structural
properties, node attributes, and label-related metrics.
These statistics are taken from \citet{platonov2023critical}
and provide additional context for understanding the experimental
results across different datasets.

\begin{table}[!hb]
    \centering
    \caption{Statistics of the five heterophilic graph benchmarks, as reported by ~\cite{platonov2023critical}.}
    \label{tab:heterophily-statistics}
    \small
    \setlength{\tabcolsep}{4pt}
    \renewcommand{\arraystretch}{0.9}
    \resizebox{0.9\linewidth}{!}{%
        \begin{tabular}{lccccc}
            \toprule
            Statistic
            & Roman-empire
            & Amazon-ratings
            & Minesweeper
            & Tolokers
            & Questions \\
            \midrule
            Nodes
            & 22662 & 24492 & 10000 & 11758 & 48921 \\
            Edges
            & 32927 & 93050 & 39402 & 519000 & 153540 \\
            Avg. degree
            & 2.91 & 7.60 & 7.88 & 88.28 & 6.28 \\
            Global clustering
            & 0.29 & 0.32 & 0.43 & 0.23 & 0.02 \\
            Avg. local clustering
            & 0.39 & 0.58 & 0.44 & 0.53 & 0.03 \\
            Diameter
            & 6824 & 46 & 99 & 11 & 16 \\
            Node features
            & 300 & 300 & 7 & 10 & 301 \\
            Classes
            & 18 & 5 & 2 & 2 & 2 \\
            Edge homophily
            & 0.05 & 0.38 & 0.68 & 0.59 & 0.84 \\
            Adjusted homophily
            & -0.05 & 0.14 & 0.01 & 0.09 & 0.02 \\
            Label informativeness
            & 0.11 & 0.04 & 0.00 & 0.01 & 0.00 \\
            \bottomrule
        \end{tabular}%
    }
\end{table}

Notably, Roman-empire exhibits relatively high Label Informativeness
(LI), and NSR achieves more pronounced performance gains on this
dataset. This observation is consistent with the discussion in the
main text that task-relevant information may influence the performance
benefits of NSR, providing additional context for understanding the
experimental results across different datasets.

\end{document}